\documentclass[letterpaper]{article} 
\usepackage{aaai2026}  
\usepackage{times}  
\usepackage{helvet}  
\usepackage{courier}  
\usepackage[hyphens]{url}  
\usepackage{graphicx} 
\usepackage{natbib}  
\usepackage{caption} 
\usepackage{algorithm}
\usepackage{algorithmic}
\usepackage{algorithmic}
\usepackage{xcolor}
\usepackage{algorithm}
\usepackage{amsmath}
\usepackage{amssymb}
\usepackage{mathtools}
\usepackage{amsthm}
\usepackage{booktabs}
\usepackage{amsthm}
\usepackage{xspace}
\usepackage{lipsum}
\usepackage{threeparttable}
\usepackage{tablefootnote}
\usepackage{colortbl}
\usepackage{lineno}
\usepackage{algorithm}
\usepackage{algorithmic}
\usepackage{multicol}
\usepackage{multirow}
\usepackage{subcaption}

\usepackage{newfloat}
\usepackage{listings}
\DeclareCaptionStyle{ruled}{labelfont=normalfont,labelsep=colon,strut=off} 
\floatstyle{ruled}
\newfloat{listing}{tb}{lst}{}
\floatname{listing}{Listing}
\newtheorem{lemma}{Lemma}
\newtheorem{proposition}{Proposition}

\theoremstyle{definition}

\theoremstyle{remark}
\newtheorem{remark}{Remark}

\title{RefiDiff: Progressive Refinement Diffusion for Efficient Missing Data Imputation}
\author{
   Md Atik Ahamed, Qiang Ye, Qiang Cheng\thanks{Corresponding author}
}

\affiliations{
    University of Kentucky\\
    {\tt\small \{atikahamed,qye3,qiang.cheng\}@uky.edu}

}

\begin{document}

\maketitle

\begin{abstract}
Missing values in high-dimensional, mixed-type datasets pose significant challenges for data imputation, particularly under Missing Not At Random (MNAR) mechanisms. Existing methods struggle to integrate local and global data characteristics, limiting performance in MNAR and high-dimensional settings. We propose an innovative framework, RefiDiff, combining local machine learning predictions with a novel Mamba-based denoising network efficiently capturing long-range dependencies among features and samples with low computational complexity. RefiDiff bridges the predictive and generative paradigms of imputation, leveraging pre-refinement for initial warm-up imputations and post-refinement to polish results, enhancing stability and accuracy. By encoding mixed-type data into unified tokens, RefiDiff enables robust imputation without architectural or hyperparameter tuning. RefiDiff outperforms state-of-the-art (SOTA) methods across missing-value settings,
demonstrating strong performance in MNAR settings and superior out-of-sample generalization. Extensive evaluations on nine real-world datasets demonstrate its robustness, scalability, and effectiveness in handling complex missingness patterns.
\end{abstract}

\begin{links}
    \link{Code}{https://github.com/Atik-Ahamed/RefiDiff}
\end{links}

\section{Introduction}

Missing values are a pervasive challenge in real-world datasets, arising from sensor failures, data corruption, or operational issues. Accurate imputation is essential for downstream analysis and modeling, particularly in high-dimensional, mixed-type datasets such as those from the UCI Machine Learning Repository~\citep{kelly2023uci} and other sources~\citep{Koklu2020MulticlassCO,pace1997sparse}.

A wide range of imputation methods has been developed, including statistical techniques, matrix completion, deep generative models, diffusion-based methods, and hybrid frameworks. These methods aim to address three canonical missingness mechanisms: Missing Completely At Random (MCAR), Missing At Random (MAR), and Missing Not At Random (MNAR). Among these, MNAR is the most challenging, as the probability of missingness depends on the missing values themselves~\citep{missing-reason-1,miracle,geron2022hands}. In contrast, MCAR is the most tractable, with missingness independent of both observed and unobserved data~\citep{mice}.

Most imputation methods fall into two fundamental paradigms: \emph{predictive} approaches, which estimate missing values through direct mappings from observed entries (e.g., via regressors or classifiers), and \emph{generative} approaches, which treat imputation as a stochastic process by modeling the underlying data distribution (e.g., via diffusion models). Predictive methods are typically accurate and efficient but deterministic, often lacking uncertainty modeling. Generative models, by contrast, offer uncertainty-aware and diverse imputations but are computationally intensive and slower.

These paradigmatic differences are further reflected in their  methodological perspectives: predictive models tend to adopt a \emph{local} view, inferring missing values using partial observations per feature, while generative models take a \emph{global} view, modeling the full data distribution. However, few methods effectively integrate these paradigms and perspectives, especially for high-dimensional, mixed-type data. This gap results in limited performance under MNAR settings, reduced robustness to distribution shifts, and sensitivity to hyperparameters. Despite their complementary strengths, predictive and generative approaches have rarely been unified into a general-purpose framework, representing a missed opportunity to jointly leverage deterministic precision and probabilistic robustness.

To address these challenges, we propose \textbf{RefiDiff}, an innovative framework that unifies predictive and generative paradigms through a progressive refinement process. RefiDiff begins with local imputations from machine learning regressors/classifiers, organized into a warm-up (pre-refinement) and polishing (post-refinement) sequence (Figure~\ref{fig:our_imputation_method}). The preliminary estimates are encoded as unified tokens and passed to a global generative stage powered by diffusion. For this, we introduce a Mamba-based denoising network~\citep{gu2023mamba,ahamed2024mambatab}, a selective state-space model with linear complexity and Transformer-level expressivity. This design combines the efficiency of predictive methods with the flexibility of diffusion models, enabling effective modeling of long-range dependencies across complex, mixed-type data.

RefiDiff outperforms state-of-the-art methods such as DIFFPUTER~\citep{zhang2025diffputer} across all three missingness scenarios (MCAR, MAR, MNAR), with particularly strong performance under MNAR. Compared to DIFFPUTER’s TabDDPM, our Mamba-based model reduces computational cost by 4$\times$ while maintaining high accuracy. Extensive evaluations on nine real-world datasets confirm RefiDiff’s robustness, scalability, and efficiency.

In summary, our main contributions are as follows:
\begin{itemize}
\item We propose a progressive refinement strategy that systematically unifies predictive and generative paradigms for robust and flexible imputation.
\item We introduce a novel Mamba-based denoising model for capturing long-range dependencies efficiently in high-dimensional, mixed-type tabular data.
\item RefiDiff is a plug-and-play solution, requiring no architectural or hyperparameter tuning across diverse datasets.
\item It achieves state-of-the-art performance under MCAR, MAR, and MNAR settings, with notable gains in MNAR scenarios.
\end{itemize}
We validate RefiDiff through comprehensive experiments and ablation studies on nine real-world datasets.

\section{Related Work}
\begin{figure*}[ht]
    \centering
    \includegraphics[width=0.85\linewidth]{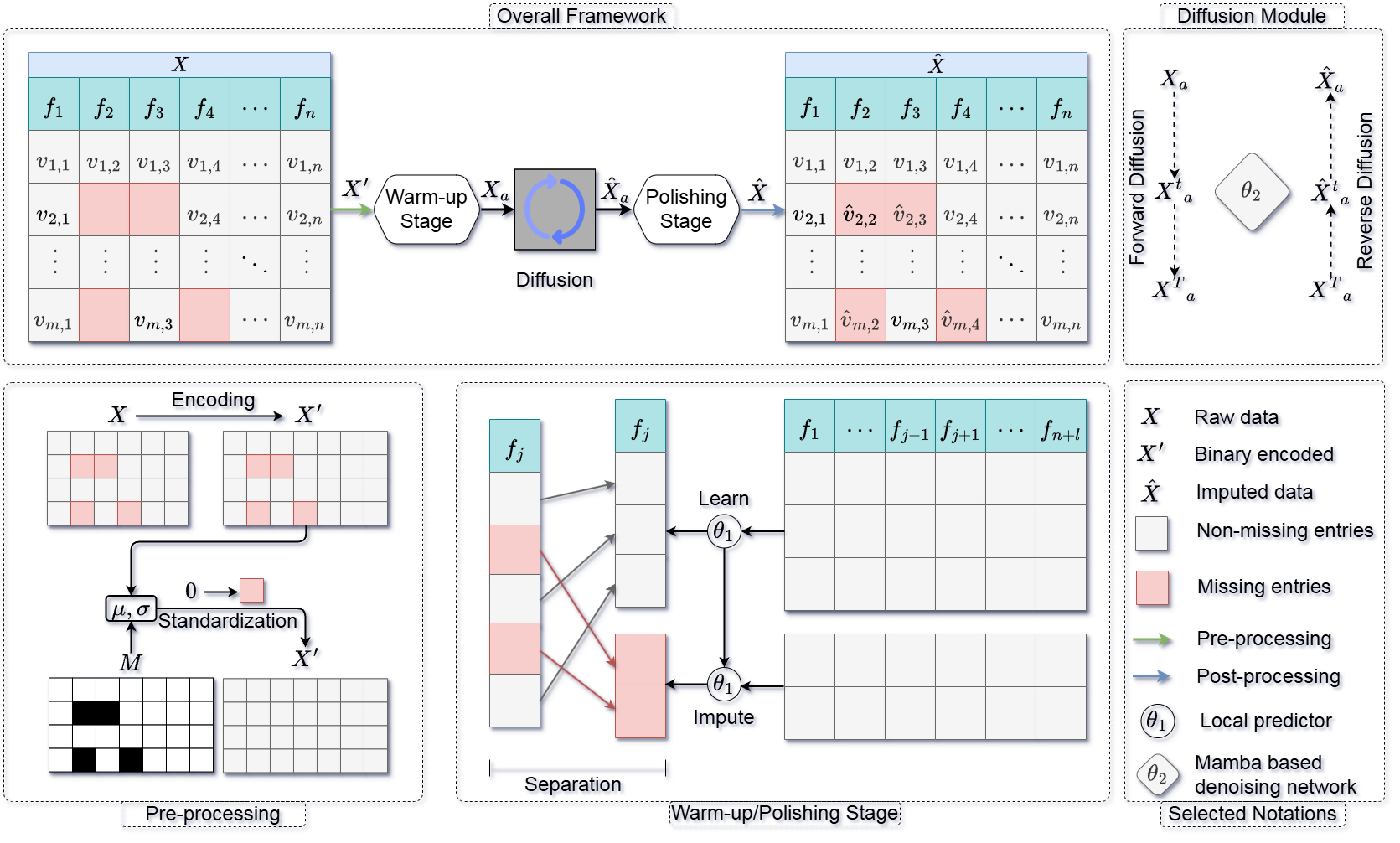}
    \caption{Overview of the proposed imputation framework. The process begins with a warm-up stage on pre-processed data, followed by a diffusion module that iteratively denoises the data. A polishing stage further enhances the imputations. Our designed denoiser $\theta_2$ will be shown in Figure~\ref{fig:denoising_network_ours}.}
    \label{fig:our_imputation_method}
    \vspace{-16pt}
\end{figure*}
The field of data imputation spans several methodological categories, which we review below.

{\bf{Traditional methods}} rely on statistical and iterative approaches. The Expectation-Maximization (EM) algorithm \citep{em,em-early} provides maximum likelihood estimates of model parameters under a pre-specified model for imputing missing data. Multiple Imputation by Chained Equations (MICE) \citep{mice} uses regression models to impute values iteratively, while K-Nearest Neighbors (KNN) imputation \citep{knn} estimates missing values based on data point similarity. MissForest \citep{missforest} employs random forests to handle mixed-type data non-parametrically. While these approaches are widely adopted for their interpretability and efficiency, they mainly leverage local relationships between observed values and often struggle with complex, high-dimensional datasets or non-linear patterns.

{\bf{Matrix completion}} techniques offer another perspective on imputation, moving beyond traditional approaches. SoftImpute \citep{softimpute} uses low-rank matrix factorization to recover missing entries, optimizing via alternating least squares. This approach is particularly effective for datasets with underlying low-rank structures, such as those in recommender systems \citep{luenberger1997optimization, kang2016top}. However, these methods may underperform when the missingness mechanism is complex or when the data do not conform to low-rank assumptions.

{\bf{Deep learning}} advances have introduced more sophisticated imputation techniques. Variational Autoencoders (VAEs) \citep{vae} and their extensions, such as MIWAE \citep{miwae} and MissVAE \citep{missvae}, model the data distribution to impute missing values, capturing complex patterns through latent representations. Generative Adversarial Networks (GANs) \citep{gan}, as seen in GAIN \citep{gain}, employ adversarial training to generate realistic imputations and have been adapted for tabular data. Normalizing flows \citep{normalizing_flows}, as explored in MCFlow \citep{mcflow}, offer exact likelihood estimation via invertible mappings, which can improve modeling of complex global distributions but their impact on imputation accuracy is indirect and architecture-dependent. These generative models excel in capturing non-linear dependencies but may require significant computational resources and careful tuning \citep{adam}.

{\bf{Diffusion models}} have recently emerged as a promising paradigm for data imputation, inspired by their success in image generation \citep{ddpm,edm,repaint}. TabDDPM \citep{tabddpm} adapts denoising diffusion probabilistic models for tabular data, modeling data generation as a reverse diffusion process. Similarly, TabCSDI \citep{tabcsdi} and MissDiff \citep{missdiff} focus on handling missing values by conditioning diffusion processes on observed data. DIFFPUTER \citep{zhang2025diffputer} integrates EM-driven diffusion for robust imputation, while DiffImpute \citep{diffimpute} leverages denoising diffusion to impute tabular data with many iterations. ForestDiff \citep{forest_diff} combines diffusion with gradient-boosted trees, offering a hybrid approach. These methods demonstrate strong performance in capturing complex global distributions but may face challenges with scalability and hyperparameter sensitivity \citep{mle_diff,vp}, and their iterative sampling nature may lead to slower inference times compared to autoregressive or flow-based models.

{\bf{Causally-aware and graph-based methods}}, such as MIRACLE \citep{miracle}, incorporate causal relationships to improve imputation accuracy, particularly in datasets with structural dependencies. Graph-based approaches, including GRAPE \citep{grape} and IGRM \citep{igrm}, model data as graphs to leverage relational information for imputation. These methods are effective in structured datasets but may require domain-specific knowledge to define graph structures accurately, leading to limitations such as scalability with large graphs or sensitivity to graph structure specification. 

{\bf{Adaptive and hybrid frameworks}} aim to combine the strengths of multiple imputation strategies. HyperImpute \citep{hyperimpute} employs automatic model selection, e.g., via meta-learning, to choose the best imputation method for a given dataset, enhancing generalizability. ReMasker \citep{du2024remasker} utilizes masked autoencoding to impute tabular data, drawing inspiration from self-supervised learning paradigms. Optimal transport-based methods, such as MOT \citep{mot} and TDM \citep{tdm}, frame imputation as a distribution matching problem, offering robust solutions for heterogeneous data. These approaches are highly flexible but may introduce additional computational overhead.

In summary, these data imputation methods have advanced the modeling of complex missingness patterns in datasets such as those from the UCI Machine Learning Repository and real-world applications. However, most struggle to simultaneously capture both global and local structures of the observed and missing entries, as well as their interdependencies. To address this, we propose a novel framework that effectively models these relationships. Our approach achieves SOTA performance under MCAR, MAR, and particularly MNAR conditions, with a significantly more efficient architecture compared to diffusion-based models like DIFFPUTER.

\section{Methodology}
\label{sec:method}
In this section, we describe each component of our framework step by step. The overall architecture, shown in Figure~\ref{fig:our_imputation_method}, consists of four major stages: Pre-processing, Warm-up Refinement, Diffusion Imputation, and Post-processing and Polishing, which we detail below.\\\\
\textbf{Preliminaries.}
We consider a dataset represented as a matrix $X \in \mathbb{R}^{m \times n}$, where $m$ is the number of samples and $n$ is the number of features. Missing entries in $X$ are indicated by a binary mask $M \in \{0,1\}^{m\times n}$, where $M_{i,j}=1$ means $X_{i,j}$ is missing and $M_{i,j}=0$ means it is observed. Our goal is to estimate all missing values $X_{i,j}$ for which $M_{i,j}=1$, using only the observed portions of the data. The missingness pattern can follow any of the standard mechanisms: MCAR, MAR, or MNAR (details are provided in Appendix A), and we consider both in-sample and out-of-sample scenarios when evaluating our method. Regardless of mechanism or data split, the core objective remains the same: to robustly impute missing entries across all three missingness settings and generalize well to unseen (out-of-sample) data.\\\\
\textbf{Pre-processing.} Real-world datasets can contain both numerical and categorical features. We first apply a type-preserving encoding and normalization pipeline to 
prepare the data for downstream modeling. Categorical features are binary-encoded, yielding $l$ additional binary indicator columns. These are concatenated with the $n$ numerical features to form an expanded matrix $X' \in \mathbb{R}^{m \times d}$ with $d = n + l$ columns. The mask $M$ is expanded correspondingly to $d$ columns (marking the new binary features as missing wherever the original categorical entry was missing). Next, we standardize each numerical feature using its mean $\mu$ and standard deviation $\sigma$ computed from the {\textit{non-missing training entries only}} (to avoid leakage of test information or missingness bias). The same $\mu$ and $\sigma$ are later used to standardize that feature in the out-of-sample set. All observed values in $X'$ are thus transformed to have zero mean and unit variance (approximately), while missing entries are temporarily filled with zeros as neutral placeholders. After pre-processing, we get a standardized matrix with normalized observed entries and zero-filled missing entries awaiting imputation.\\\\
\textbf{Warm-up Refinement.}
Before invoking the diffusion model, we perform a warm-up imputation that provides initial guesses for all missing values in $X'$. The idea is to leverage fast, local models to ease the subsequent global modeling task, by filling in plausible values; especially in heavy-MNAR or out-of-distribution cases, we reduce the burden on the diffusion model, which would otherwise have to start from arbitrary initializations (e.g., zero-filled vectors).
We adopt a simple column-wise regression strategy inspired by methods like MICE, but done in a single pass. Concretely, for each feature column $f_j$ (where $j=1,\dots,d$), we use the current partially imputed data to train a lightweight predictive model $\theta_1^{(j)}$ (e.g., a default XGBoost or CatBoost regressor/classifier). The model $\theta_1^{(j)}: \mathbb{R}^{d-1} \to \mathbb{R}$ is trained to predict feature $f_j$ from all other features: we take all samples $i$ where $f_j$ is observed ($M_{i,j}=0$) and use $(X'_{i,\setminus j},\,X'_{i,j})$ as input-output pairs to fit $\theta_1^{(j)}$. Once trained, we apply $\theta_1^{(j)}$ to estimate the missing entries in column $j$, i.e. for each $i$ with $M_{i,j}=1$ we set $({X_a})_{i,j} := \theta_1^{(j)}(X'_{i,\setminus j})$. We then discard $\theta_1^{(j)}$ and move to the next column. This process sweeps through the feature set $j=1$ to $d$ exactly once. Unlike iterative methods (EM, MICE) or recent models like DIFFPUTER that perform multiple passes, our warm-up produces a complete imputed dataset $X_a$ in {\textit{one pass}}, greatly improving efficiency.

{\textit{Properties of One-Pass Imputation}}: This refinement has three key properties: (i) {\textit{Non-overwriting}}: all originally observed entries remain unchanged in $X_a$; (ii) {\textit{Well-defined mapping}}: each missing entry is imputed by a deterministic function of that sample’s other features (through the learned model for that column); and (iii) {\textit{One-pass completion}}: each feature is processed once and each missing value is filled exactly one time, so the procedure terminates after a single sweep. These outcomes are guaranteed by construction, and we formalize them in Proposition 2 (Appendix B) and provide a proof there for completeness. In summary, the warm-up stage yields an initial imputed matrix $X_a$ that preserves all observed data and provides reasonable first-fill values for all missing entries.\\\\
\textbf{Diffusion Module.}
After warm-up, we feed the refined data $X_a$ into a generative diffusion model to {\textit{globally refine}} the imputation. The diffusion module treats the entire dataset (or each data sample) as input and performs a denoising diffusion process conditioned on the mask $M$. The core idea is to progressively replace the initial guesses in $X_a$ with more accurate values by sampling from a model of the joint data distribution, all while exactly preserving the observed entries. We adopt a continuous-time diffusion process: starting from $X_0 = X_a$, we gradually add Gaussian noise to corrupt $X_0$ into a noisy version $X_t$ (for $t$ increasing to some $T$), according to a noise schedule $\sigma(t)$ (details in Appendix C). A denoising network $\theta_2$ is trained to reverse this process, i.e., to predict the instantaneous noise at each step $t$ and push $X_t$ back toward the clean $X_0$. Crucially, $\theta_2$ is given access to the mask $M$ and is designed to leave observed components untouched: at each denoising step, we clamp $X_{t}$ at the observed coordinates to equal $X_{\!a}$, and update only the missing coordinates. This way, the diffusion never alters ground-truth data and focuses its denoising only on unknown parts. We use a lightweight yet expressive architecture for $\theta_2$, illustrated in Figure \ref{fig:denoising_network_ours}. 

\begin{figure}[t]
    \centering
\includegraphics[width=0.6\linewidth]{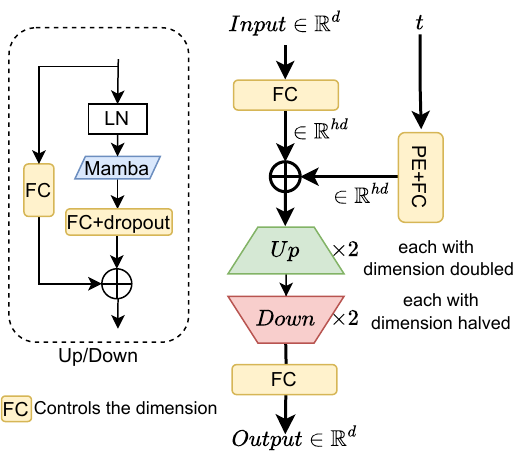}
    \caption{Illustration of the denoising network $\theta_2$, featuring a diamond-shaped structure with Mamba-based residual Up/Down blocks.}
    \label{fig:denoising_network_ours}
\vspace{-15pt}
\end{figure}
It has a symmetric ``diamond'' structure with two residual up-sampling blocks followed by two down-sampling blocks (each block built around the Mamba state-space layer to capture long-range dependencies efficiently). This design gives Transformer-like expressivity with linear complexity, enabling us to model high-dimensional data (many features and samples) without quadratic cost. Each block applies dimension doubling or halving (for up/down steps) with skip connections, fully connected layers (FC), positional encoding (PE), layer normalization, and dropout, yielding a robust multi-scale denoiser.

During training, we optimize $\theta_2$ using the EDM loss~\cite{karras2022elucidating}. Specifically, the loss is of the form:
\[
\mathcal{L}_{\text{SM}}(\theta_2) = \mathbb{E}_{X_0, \varepsilon, t} \left[ \left\| \theta_2(X_t, t, M) - \nabla_{X_t} \log p(X_t \mid X_0) \right\|_2^2 \right],
\]

This loss weights the denoising objective according to noise magnitude (so the model learns to handle all corruption levels), thus handling noise-adaptivity better than standard score-matching loss. 
At inference, we run the reverse diffusion: initialize $X_T$ as Gaussian noise (with observed entries still clamped to $X_{\!a}$), and iteratively denoise from $t=T$ down to $t=0$ using $\theta_2$. We perform $N$ stochastic runs of this reverse process for each input and average the results to obtain a final imputed output $\hat{X}_a$. This ensembling over multiple diffusion trajectories improves the robustness and stability of the imputation. 

We adopt the Variance Exploding (VE) SDE formulation for our diffusion process, as it provides higher noise levels at earlier timesteps, which empirically leads to improved sample diversity and better exploration of the missing-value space. VE SDE is particularly advantageous in imputation tasks under MNAR conditions, where the model benefits from starting from a more randomized state to avoid overfitting to local patterns. Unlike VP SDEs (e.g., DDPM), which gradually corrupt data with bounded noise, VE allows a more aggressive corruption and reconstruction process, which improves robustness in our experiments. We found this choice to yield more stable and diverse imputations. 

To provide a theoretical foundation for the validity of our masked diffusion-based imputation, we establish a conditional sampling guarantee under ideal training assumptions: 
The RefiDiff reverse process recovers the true conditional distribution of the missing values given the observed ones. This result and proof can be found in Appendix D.

To further strengthen the theoretical foundation of RefiDiff, we provide a quantitative bound on its imputation accuracy. Specifically, we analyze how closely the reverse diffusion process recovers the true conditional distribution $p(\mathbf{x}^{\text{mis}} \mid \mathbf{x}^{\text{obs}})$, in terms of the quality of the learned denoising score function, the discretization of the diffusion process, and the number of stochastic samples used. This result is formalized in the proposition below:
\begin{proposition}[Approximation Bound for RefiDiff]
\label{prop:approximation_bound_brief}
Under mild regularity conditions, the KL divergence between the RefiDiff-imputed distribution $\hat{p}(\mathbf{x}^{\text{mis}} \mid \mathbf{x}^{\text{obs}})$ and the true conditional distribution is bounded by:
\[
\mathrm{KL}\big(\hat{p}(\mathbf{x}^{\text{mis}} \mid \mathbf{x}^{\text{obs}}) \,\|\, p(\mathbf{x}^{\text{mis}} \mid \mathbf{x}^{\text{obs}})\big) \leq C_1 T \varepsilon_\theta^2 + C_2 \ \delta t + C_3 \  \frac{1}{N},
\]
where $\varepsilon_\theta$ is the error of the learned score function, $\delta t$ is the diffusion step size, $N$ is the number of reverse diffusion trajectories averaged, and $C_1, C_2, C_3$ are constants independent of the data.
\end{proposition}
The full version of this proposition and its brief proof are provided in Appendix E. This bound shows that RefiDiff's imputation converges to the Bayes-optimal conditional mean as the score model improves, the diffusion steps are refined, and sufficient averaging is performed. Compared to prior qualitative results, this quantitative guarantee helps justify our design choices, including the use of ensembling, careful time discretization, and the progressive refinement pipeline. This bound also motivates our specific design strategies, including multi-trajectory averaging and a progressive denoising approach.\\\\
\textbf{Post-processing and Polishing.} 
Finally, we apply a brief post-processing and polishing step. After the polishing stage, we invert the earlier standardization and encoding: the entries of $\hat{X}$ are de-standardized by multiplying by $\sigma$ and adding $\mu$ for each feature, and binary-encoded categorical columns are decoded back to the original category labels. This yields the final imputed data matrix $\hat{X}$ in the original feature space. As an additional polishing, 
optionally, we perform a final sweep of the lightweight column-wise models $\theta_1^{(j)}$ on $\hat{X}_a$ to refine any residual discrepancies introduced during diffusion.
This polishing may help because diffusion may introduce slight distributional noise, a quick predictive pass can adjust the estimates to better conform to typical feature distributions. This polishing re-uses the same one-pass procedure described earlier, now applied to $\hat{X}_a$. The polishing ensures that $\hat{X}$ retains all originally observed entries (by construction) and that all imputed values are as reasonable and consistent as possible feature-wise. After post-processing, $\hat{X}$ is ready for downstream use. It contains no missing entries, preserves information from $X$, and reflects a blend of local predictive accuracy and global generative coherence.

Collectively, these stages (pre-processing, warm-up refinement, diffusion-based global refinement, and post-processing) form a theoretically grounded and practically robust imputation pipeline for diverse missingness patterns.
\section{Experiments}
In this section, we evaluate our framework against existing imputation methods on nine benchmark datasets, following the setup of DIFFPUTER~\cite{zhang2025diffputer}. Results are reported under three missingness mechanisms: MNAR, MCAR, and MAR, adhering to established protocols~\cite{zhang2025diffputer,du2024remasker}. Unlike prior work~\cite{du2024remasker,zhang2025diffputer} that omits some dataset-mechanism combinations, we conduct a complete evaluation for all three types. To ensure fairness, all methods are tested with identical 10 random masks. For baselines, we use official implementations with recommended hyperparameters. Performance is measured only on missing entries, as observed values remain unchanged.\\
\begin{table*}[t]
\centering

\setlength{\tabcolsep}{1mm}
\begin{tabular}{lcccc|cccc|cccc|c}
\toprule
\multirow{3}{*}{Method} 
& \multicolumn{4}{c}{\textbf{MNAR}} 
& \multicolumn{4}{c}{\textbf{MCAR}} 
& \multicolumn{4}{c}{\textbf{MAR}}  & \multirow{3}{*}{Rank ($\downarrow$)} \\
& \multicolumn{2}{c}{In-Sample} & \multicolumn{2}{c}{Out-of-Sample}
& \multicolumn{2}{c}{In-Sample} & \multicolumn{2}{c}{Out-of-Sample}
& \multicolumn{2}{c}{In-Sample} & \multicolumn{2}{c}{Out-of-Sample}\\
& MAE & RMSE & MAE & RMSE 
& MAE & RMSE & MAE & RMSE 
& MAE & RMSE & MAE & RMSE\\
\midrule
EM & 42.42 & 82.30 & 46.13 & 107.00 & 38.70 & 67.42 & 40.93 & 89.71 & 42.57 & 82.63 & 43.72 & 91.08 & 5.42 \\
MIWAE & 68.21 & 121.21 & 65.99 & 110.34 & 62.55 & 102.88 & 62.59 & 101.61 & 68.24 & 122.50 & 65.85 & 112.69 & 10.08 \\
GAIN & 75.84 & 126.83 & 73.26 & 117.78 & 67.02 & 104.05 & 66.38 & 104.74 & 82.85 & 138.06 & 77.00 & 126.87 & 11.75 \\
SoftImpute & 59.82 & 103.66 & 59.68 & 99.27 & 52.74 & 84.01 & 53.95 & 87.48 & 60.80 & 104.45 & 59.30 & 99.16 & 7.92 \\
MICE & 69.01 & 124.49 & 70.83 & 614.63 & 65.66 & 95.08 & 65.60 & 94.31 & 69.47 & 108.12 & 73.17 & 1062.95 & 10.67 \\
MIRACLE & 54.54 & 110.86 & 55.86 & 110.21 & 48.47 & 89.45 & 51.24 & 99.10 & 59.48 & 117.85 & 64.17 & 116.85 & 8.67 \\
KNN & 55.95 & 108.99 & 51.30 & 92.93 & 46.42 & 82.48 & 46.39 & 80.90 & 43.22 & 88.43 & 46.45 & 87.82 & 6.42 \\
MissForest & 46.50 & 89.51 & 46.13 & 83.89 & 43.01 & 75.27 & 43.05 & 74.18 & 48.21 & 94.38 & 46.05 & 86.80 & 5.58 \\
HyperImpute & 37.95 & 79.61 & 38.45 & 104.61 & 34.51 & 65.22 & 36.38 & 158.27 & 38.62 & 80.81 & 38.89 & 111.65 & 4.67 \\
DIFFPUTER & 37.27 & 86.86 & 34.54 & 72.73 & 31.72 & 63.49 & 31.39 & 61.85 & 39.15 & 90.95 & 35.32 & 76.09 & \underline{2.67} \\
ReMasker & 39.66 & 80.23 & 39.52 & 74.14 & 35.84 & 65.19 & 35.84 & 64.15 & 38.39 & 78.82 & 38.06 & 74.90 & 3.00 \\
\rowcolor{blue!10} Ours & 34.49 & 78.83 & 34.38 & 70.12 & 31.41 & 63.16 & 32.20 & 63.11 & 34.52 & 78.22 & 34.43 & 73.82 & \textbf{1.17} \\
\bottomrule
\end{tabular}
\caption{Performance comparison across all methods and settings for numerical columns. We report MAE and RMSE for in-sample and out-of-sample imputation under three missingness mechanisms: MNAR, MCAR, and MAR. Lower values indicate better performance. The best and second-best average ranks are highlighted in \textbf{bold} and \underline{underline}, respectively.}
\label{tab:main_reg}
\end{table*}
\begin{table}[t]
\centering

\setlength{\tabcolsep}{1mm}

\begin{tabular}{lcc|cc|cc|c}
\toprule
\multirow{2}{*}{Method} 
& \multicolumn{2}{c}{\textbf{MNAR}} 
& \multicolumn{2}{c}{\textbf{MCAR}} 
& \multicolumn{2}{c}{\textbf{MAR}} & Rank \\
& IS & OOS
& IS & OOS & IS & OOS & $\downarrow$ \\

\midrule
EM & 58.15 & 57.77 & 58.26 & 58.08 & 54.77 & 57.81 & 4.67 \\
MIWAE & 41.74 & 42.06 & 41.91 & 42.07 & 40.57 & 41.99 & 12.00 \\
GAIN & 48.31 & 46.98 & 48.86 & 47.62 & 44.99 & 47.21 & 9.33 \\
SI & 47.86 & 47.11 & 47.25 & 47.35 & 45.67 & 46.84 & 9.67 \\
MICE & 45.52 & 45.53 & 45.41 & 45.54 & 43.33 & 45.84 & 11.00 \\
Miracle & 54.97 & 54.57 & 55.22 & 54.46 & 49.84 & 51.43 & 7.33 \\
KNN & 53.77 & 54.09 & 54.14 & 54.01 & 53.37 & 55.65 & 7.50 \\
MF & 55.34 & 55.25 & 55.76 & 55.69 & 52.25 & 57.55 & 6.00 \\
HI & 59.69 & 58.96 & 60.14 & 58.59 & 54.42 & 57.04 & 4.50 \\
DP & 60.07 & 60.49 & 60.26 & 60.36 & 57.24 & 60.85 & 3.00 \\
RM & 63.01 & 62.92 & 63.25 & 63.04 & 59.88 & 64.43 & \underline{1.83} \\
\rowcolor{blue!10} Ours & 63.19 & 63.08 & 63.56 & 63.20 & 60.05 & 64.35 & \textbf{1.17} \\
\bottomrule
\end{tabular}
\caption{Performance comparison across all methods and settings for categorical columns. We report accuracy for in-sample (IS) and out-of-sample (OOS) imputation under all missingness mechanisms. Higher values indicate better performance. The best and second-best average ranks are highlighted in \textbf{bold} and \underline{underline}, respectively.}
\label{tab:main_acc}
\vspace{-20pt}
\end{table}
\begin{figure}[t]
    \centering
    
    \begin{subfigure}[t]{0.48\linewidth}
    
        \centering        
        \includegraphics[width=\linewidth]{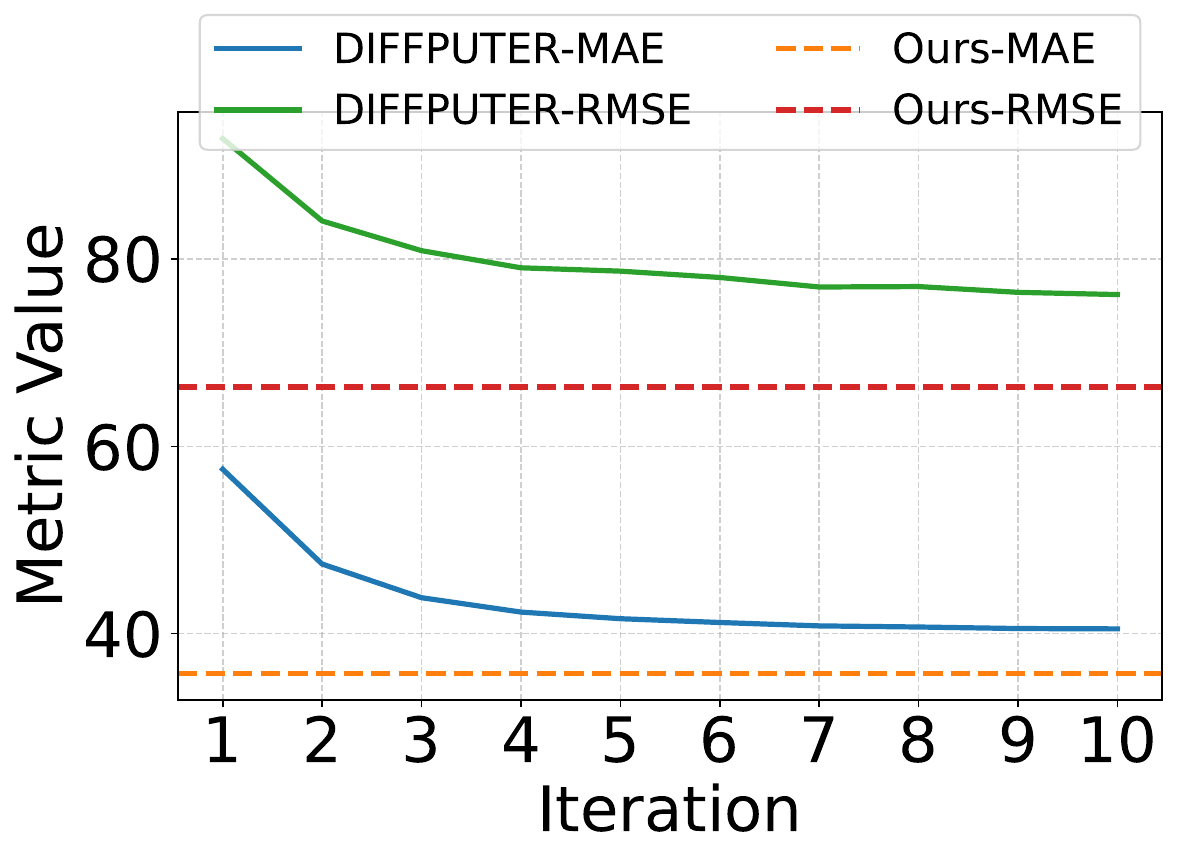}
        \caption{California}
        \label{subfig:california_iter}
    \end{subfigure}
    \hfill
    \begin{subfigure}[t]{0.48\linewidth}
        \centering
        \includegraphics[width=\linewidth]{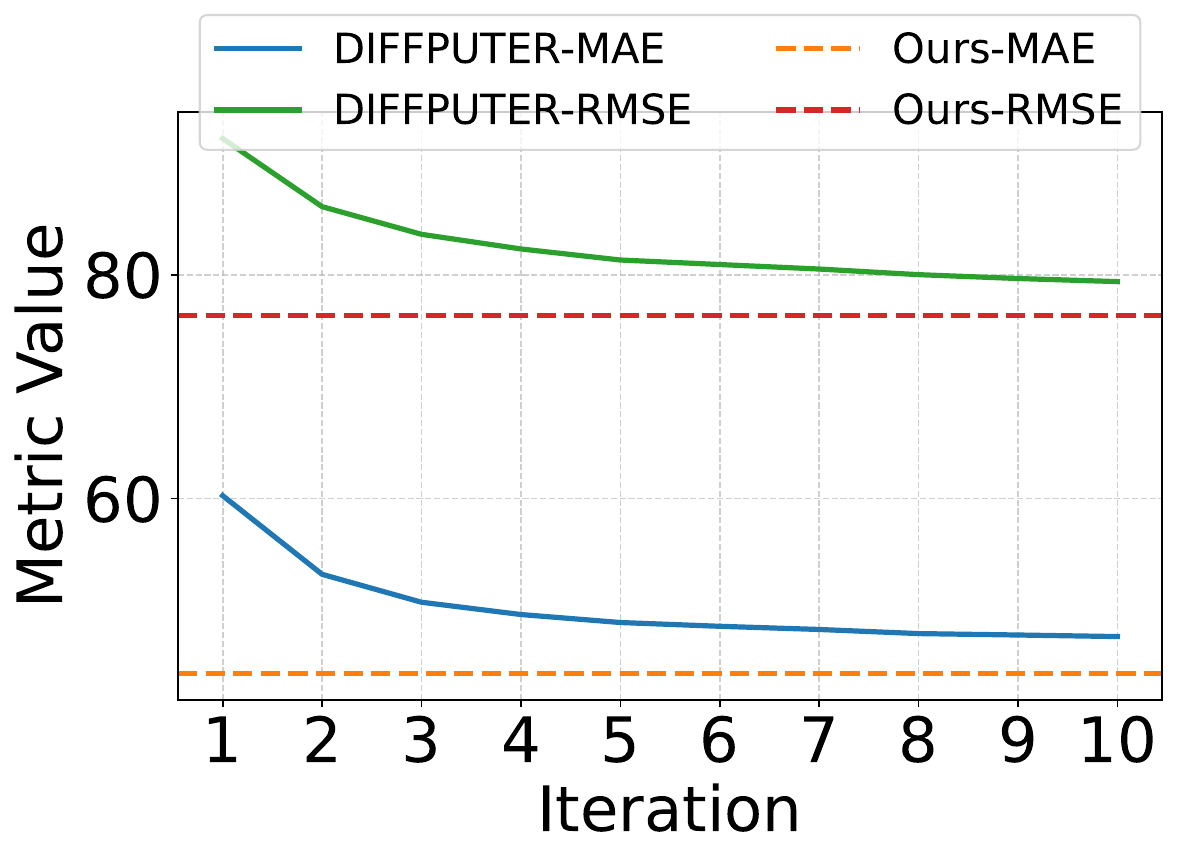}
        \caption{Magic}
        \label{subfig:magic_iter}
    \end{subfigure}
    \hfill
    \begin{subfigure}[t]{0.48\linewidth}
        \centering
        \includegraphics[width=\linewidth]{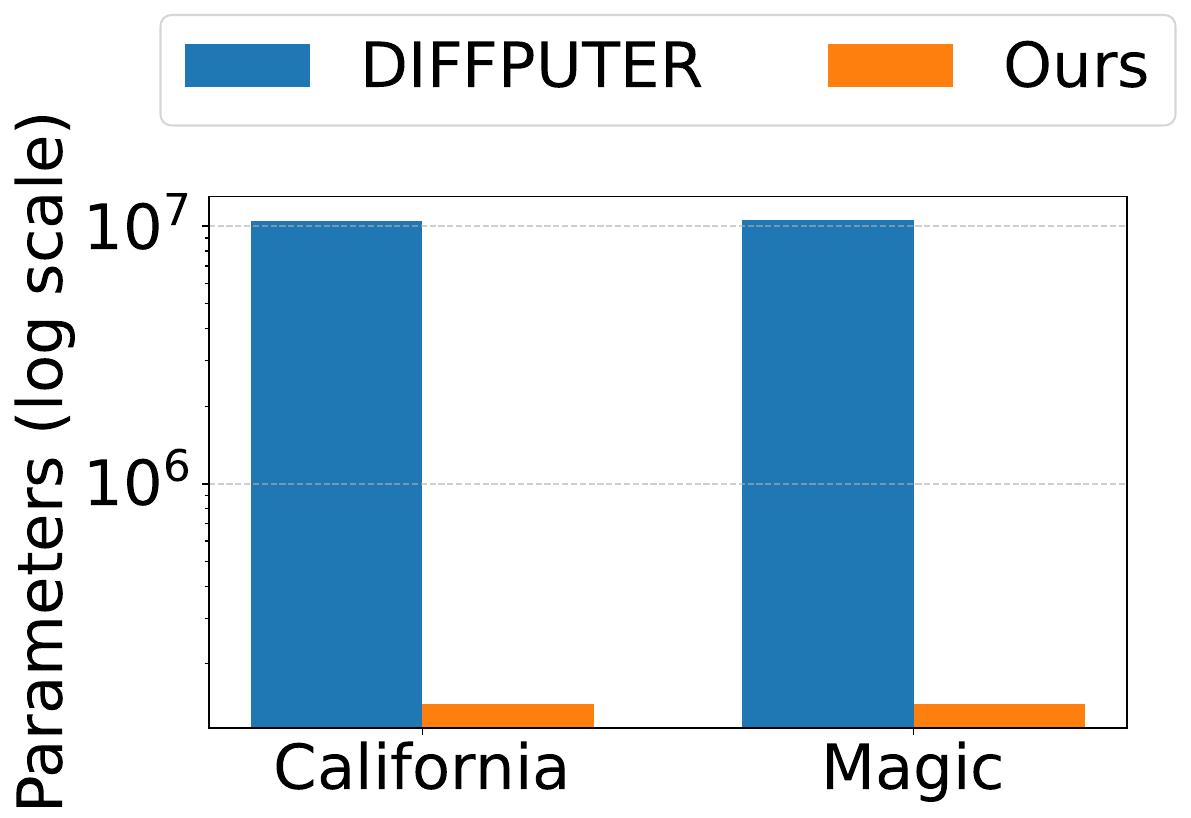}
        \caption{Parameters}
        \label{subfig:denoising_param}
        
    \end{subfigure}
    \hfill
    \begin{subfigure}[t]{0.48\linewidth}
        \centering
        \includegraphics[width=\linewidth]{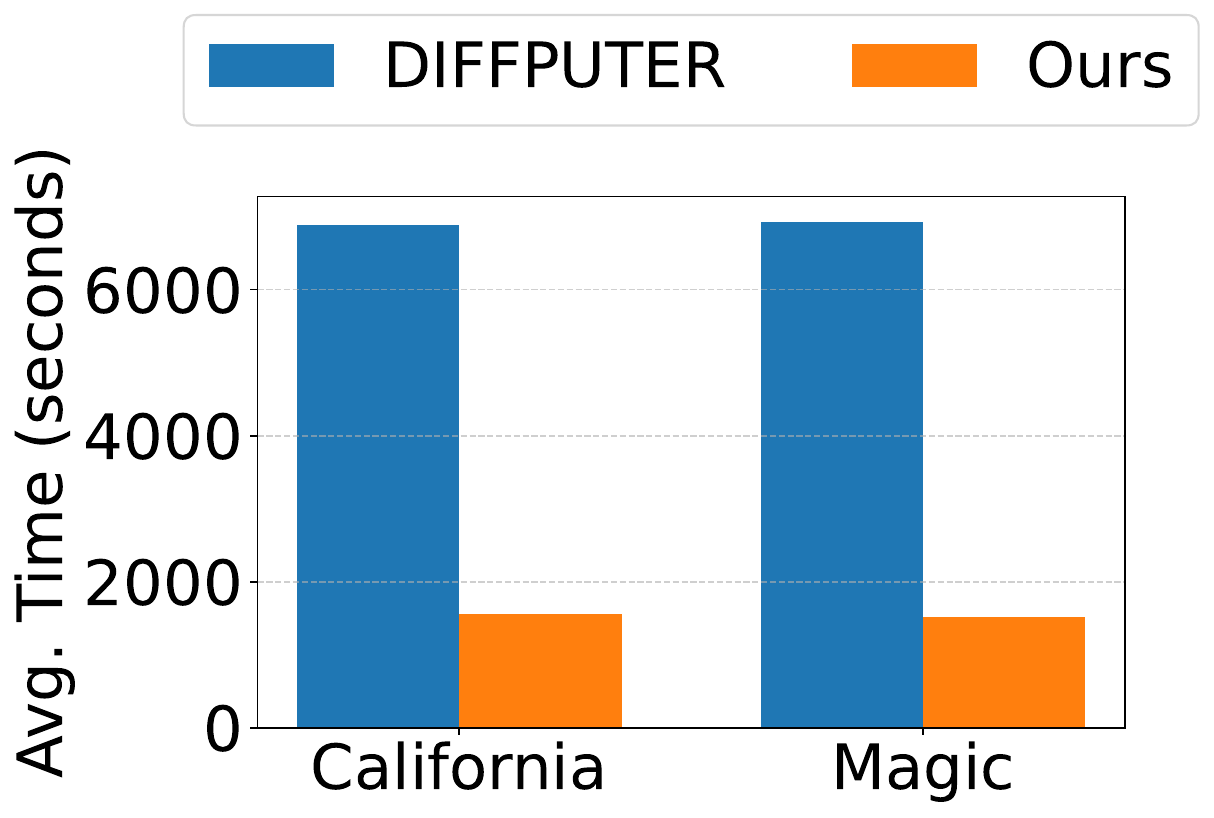}
        \caption{Run time}
        \label{subfig:run_time}
    
    \end{subfigure}
    \caption{Comparison between DIFFPUTER and our method on two datasets (California and Magic) under the MNAR setting. (a) and (b) show in-sample MAE and RMSE over iterations. (c) compares denoising network parameter counts. (d) presents average runtime over 10 random masks.}
    \label{fig:efficiency}
    \vspace{-10pt}
\end{figure}
\begin{figure}[t]
    \centering
    
    \begin{subfigure}[t]{\linewidth}
    
        \centering        
        \includegraphics[width=\linewidth]{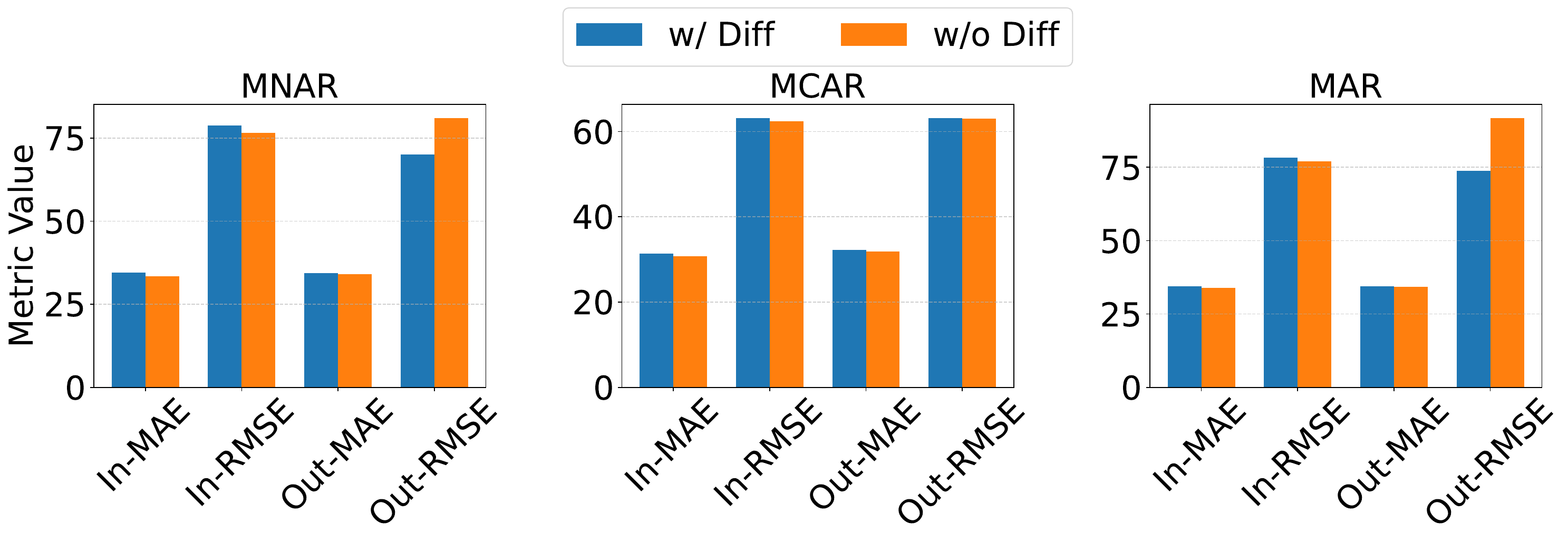}
        \caption{With and without diffusion}
        \label{subfig:ablation_diff}
    \end{subfigure}
    \hfill
    \begin{subfigure}[t]{\linewidth}
        \centering
        \includegraphics[width=\linewidth]{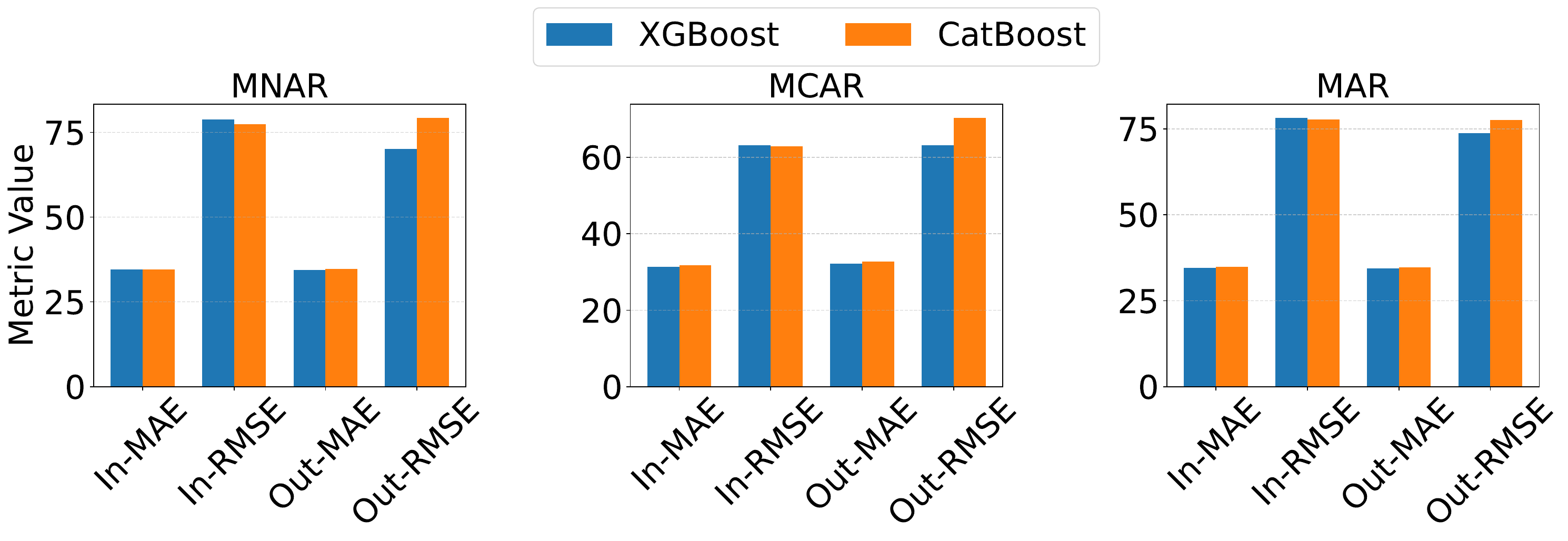}
        \caption{Ablation on regressor choice}
        \label{subfig:ablation_regressor}
    \end{subfigure}
    
    \caption{Ablation experiments of our proposed framework in various settings and components.}
    \label{fig:ablation}
    \vspace{-15pt}
\end{figure}

\noindent{\textbf{Datasets and Evaluation Metrics.}}
We evaluate all methods on nine real datasets~\citep{zhang2025diffputer}: five with only numerical features (California, Magic, Bean, Gesture, Letter) and four with mixed features (Default, News, Adult, Shoppers). Performance is measured by MAE and RMSE for numerical attributes and accuracy for categorical ones. Dataset details are in Appendix F.\\
\textbf{Data Processing.}
All features (numerical and categorical) are merged into a unified data matrix and standardized (Section~\ref{sec:method} Pre-processing). We use 70\% of the data as in-sample and 30\% as out-of-sample, introducing 30\% missingness via binary masks for MNAR, MCAR, and MAR. Separate masks are used for in- and out-of-sample evaluations. Categorical features are binary-encoded and zero-padded for uniform width, while missing values are zero-initialized and masked to prevent access to ground truth during training. Following DIFFPUTER, categorical columns are de-standardized and decoded for accuracy, while numerical columns remain standardized for MAE and RMSE. Full de-standardization can be applied if needed. \\
\textbf{Baselines and Implementation.}
For benchmarking, we compare our model, RefiDiff with default structures and hyperparameters, against a carefully selected group of eleven baselines, representing classical and SOTA models. The baselines include EM~\cite{em}, MIWAE~\cite{miwae}, GAIN~\cite{gain}, SoftImpute (SI)~\cite{softimpute}, MICE~\cite{mice}, MIRACLE (Miracle)~\cite{miracle}, KNN, MissForest (MF)~\cite{missforest}, HyperImpute (HI)~\cite{hyperimpute}, DIFFPUTER (DP)~\cite{zhang2025diffputer}, and ReMasker (RM)~\cite{du2024remasker}, which are representative of traditional iterative, matrix completion-based imputation techniques, deep learning-based methods, adversarial learning, diffusion methods. We use official implementations and recommended settings for all baselines, evaluating them under the same masks, metrics, and data splits for fairness. Implementation and hyperparameter details are in Appendix G.\\
\textbf{Result Analysis.} We evaluate all methods under three missingness mechanisms using in- and out-of-sample metrics. Tables~\ref{tab:main_reg} and \ref{tab:main_acc} show averaged results across datasets and masks, with detailed results in Appendix K. 

Table~\ref{tab:main_reg} reports MAE/RMSE (scaled by $10^{-2}$ for better readability) for numerical imputation, and Table~\ref{tab:main_acc} shows categorical accuracy. RefiDiff achieves the best or second-best performance in most cases, with an average rank of 1.17, outperforming SOTA baselines like DIFFPUTER, ReMasker, and HyperImpute.

In Table~\ref{tab:main_reg}, our method achieves consistently lower MAE and RMSE across all missingness types, demonstrating strong generalization for both numerical and categorical features. It notably surpasses SOTA baselines, particularly in the challenging MNAR scenario, due to our effective denoising and refinement stages. Our model shows statistically significant improvement over baseline models. We validate our results with statistical significance tests in Appendix H.

Figure~\ref{fig:efficiency} compares our method with DIFFPUTER in terms of convergence, complexity, and runtime. Figures~\ref{subfig:california_iter}-\ref{subfig:magic_iter} show in-sample MAE and RMSE for California and Magic datasets under MNAR. While DIFFPUTER improves gradually over iterations, our method achieves strong results without iterations, showing its ability to produce high-quality imputations efficiently. Figure~\ref{subfig:denoising_param} shows that our denoising network uses far fewer parameters than DIFFPUTER, showing both effectiveness and memory efficiency.

Figure~\ref{subfig:run_time} compares average runtime over 10 random masks. DIFFPUTER’s iterative process makes it much slower, while our method is about four times faster with lower computational cost. Both methods were tested on the same setup (NVIDIA V100 GPU, 8 CPU cores, 32 GB RAM). These results confirm that our design is both accurate and efficient, with variation analysis in Appendix I, further verifying RefiDiff’s stability and generalization.

\section{Ablation Study and Sensitivity Analysis}
To analyze the contributions of individual components, we perform ablation and sensitivity studies. The ablation examines the impact of key design choices, such as the diffusion module and regression model, on imputation performance. Sensitivity analysis explores the effect of varying the denoising network architecture and the number of sampling trials in reverse diffusion. Together, these evaluations reveal the framework's robustness, efficiency, and generalization.

\textbf{Effectiveness of the Diffusion Module.}
To evaluate the impact of the diffusion module within our framework, we conduct an ablation study comparing the full model ("w/ Diff") to a reduced variant that includes only the warm-up and polishing stages ("w/o Diff"). As shown in Figure~\ref{subfig:ablation_diff}, we report in-sample and out-of-sample MAE and RMSE under all three missingness mechanisms: MCAR, MAR, and MNAR. While the "w/o Diff" model achieves slightly lower in-sample MAE in the MCAR setting, it performs notably worse in out-of-sample RMSE, especially under MAR and MNAR. Specifically, the out-of-sample RMSE increases from 73.82 to 91.80 under MAR and from 70.12 to 81.07 under MNAR when diffusion is removed. These results highlight that the diffusion module plays a vital role in improving generalization and capturing complex feature interactions, particularly under more difficult missingness scenarios.

\textbf{Ablation on Regressor Choice.} The warm-up and polishing stages in RefiDiff rely on a regression model to estimate missing numerical values. To evaluate the sensitivity of our framework to the choice of regressor, we compare XGBoost~\cite{chen2016xgboost}, which serves as the default in our main experiments, with CatBoost~\cite{dorogush2018catboost} as an alternative. As shown in Figure~\ref{subfig:ablation_regressor}, both regressors perform competitively across all missingness types and evaluation metrics. While CatBoost yields slightly better in-sample RMSE under the MCAR setting, XGBoost consistently achieves lower out-of-sample MAE and RMSE, particularly in the more challenging MAR and MNAR scenarios. These results justify our use of XGBoost as the default regressor, as it provides a favorable balance between accuracy and robustness. Importantly, both variants outperform existing SOTA baselines, reinforcing the generalizability and effectiveness of our overall framework.

Additional experiments in Appendix J further demonstrate RefiDiff's versatility. When integrated into DIFFPUTER as a replacement denoiser (Figure 8 in Appendix J), our $\theta_2$ yields comparable performance, underscoring its plug-and-play capability. We also analyze the relationship between imputation performance and the Mamba block's hidden dimension size and its selective modeling of dependencies (Figures 9 and 10 in Appendix J), providing insights into optimal model configuration and the mechanism behind its effectiveness.

\section{Conclusion}
This paper presents RefiDiff, a framework for robust data imputation in high-dimensional, mixed-type datasets with complex missingness. RefiDiff bridges gaps in existing methods by integrating local and global data characteristics via a progressive pre- and post-refinement strategy. Locally, it uses machine learning predictions, while globally, a Mamba-based denoising network captures feature and sample dependencies~\citep{ahamed2024mambatab}. Pre-refinement generates initial imputations, refined post-hoc for accuracy and stability, enabling tuning-free imputation across diverse datasets. By encoding mixed-type data into unified tokens, RefiDiff ensures robust performance. Evaluations on nine real-world datasets show RefiDiff outperforms state-of-the-art methods in MCAR, MAR, and MNAR scenarios, excelling in MNAR with 4x faster training than the DIFFPUTER approach. Its efficiency and user-friendly design make RefiDiff ideal for practical applications. While binary encoding ensures compatibility with continuous diffusion, future work could explore native treatments of categorical variables to improve semantic fidelity. We also plan to extend RefiDiff to streaming data, adaptive refinement for sparse datasets, and applications in domains like healthcare and finance.
\section{Acknowledgments}
This work was supported in part by the NSF under Grants IIS 2327113 and ITE 2433190; and the NIH under Grants P30AG072946. We would like to thank the NSF support for AI research resources with NAIRR240219, Jetstream2, PSC, and the University of Kentucky Center for Computational Sciences and Information Technology Services Research Computing for their support and use of the LCC.
\bibliography{2.references}

@inproceedings{
zhang2025diffputer,
    title={DiffPuter: An {EM}-Driven Diffusion Model for Missing Data Imputation},
    author={Hengrui Zhang and Liancheng Fang and Qitian Wu and Philip S. Yu},
    booktitle={The Thirteenth International Conference on Learning Representations},
    year={2025},
    url={https://openreview.net/forum?id=3fl1SENSYO}
}

@inproceedings{
du2024remasker,
    title={ReMasker: Imputing Tabular Data with Masked Autoencoding},
    author={Tianyu Du and Luca Melis and Ting Wang},
    booktitle={The Twelfth International Conference on Learning Representations},
    year={2024},
    url={https://openreview.net/forum?id=KI9NqjLVDT}
}

@misc{gesture,
  author       = {Madeo, Renata and Wagner, Priscilla and Peres, Sarajane},
  title        = {Gesture Phase Segmentation},
  year         = {2013},
  publisher    = {UCI Machine Learning Repository},
  doi          = {10.24432/C5Z32C},
  note         = {{DOI}: https://doi.org/10.24432/C5Z32C}
}

@misc{letter,
  author       = {Slate, David},
  title        = {{Letter Recognition}},
  year         = {1991},
  howpublished = {UCI Machine Learning Repository},
  note         = {{DOI}: https://doi.org/10.24432/C5ZP40}
}

@article{pace1997sparse,
  title={Sparse spatial autoregressions},
  author={Pace, R Kelley and Barry, Ronald},
  journal={Statistics \& Probability Letters},
  volume={33},
  number={3},
  pages={291--297},
  year={1997},
  publisher={Elsevier}
}

@book{geron2022hands,
  title={Hands-on Machine Learning with Scikit-Learn, Keras, and TensorFlow: Concepts, Tools, and Techniques to Build Intelligent Systems},
  author={G{\'e}ron, Aur{\'e}lien},
  year={2022},
  publisher={O'Reilly Media, Inc.}
}

@misc{magic,
  author       = {Bock, R.},
  title        = {{MAGIC Gamma Telescope}},
  year         = {2004},
  howpublished = {UCI Machine Learning Repository},
  note         = {{DOI}: https://doi.org/10.24432/C52C8B}
}

@article{Koklu2020MulticlassCO,
  title={Multiclass classification of dry beans using computer vision and machine learning techniques},
  author={Murat Koklu and Ilker Ali {\"O}zkan},
  journal={Comput. Electron. Agric.},
  year={2020},
  volume={174},
  pages={105507},
  url={https://api.semanticscholar.org/CorpusID:219762890}
}

@misc{adult,
  author       = {Becker, Barry and Kohavi, Ronny},
  title        = {{Adult}},
  year         = {1996},
  howpublished = {UCI Machine Learning Repository},
  note         = {{DOI}: https://doi.org/10.24432/C5XW20}
}

@misc{default,
  author       = {Yeh, I-Cheng},
  title        = {{Default of Credit Card Clients}},
  year         = {2009},
  howpublished = {UCI Machine Learning Repository},
  note         = {{DOI}: https://doi.org/10.24432/C55S3H}
}

@misc{shoppers,
  author       = {Sakar, C. and Kastro, Yomi},
  title        = {{Online Shoppers Purchasing Intention Dataset}},
  year         = {2018},
  howpublished = {UCI Machine Learning Repository},
  note         = {{DOI}: https://doi.org/10.24432/C5F88Q}
}

@misc{news,
author         = {Fernandes, Kelwin and Vinagre, Pedro and Cortez, Paulo                        and Sernadela, Pedro},
  title        = {{Online News Popularity}},
  year         = {2015},
  howpublished = {UCI Machine Learning Repository},
  note         = {{DOI}: https://doi.org/10.24432/C5NS3V}
}

@inproceedings{pytorch,
  title={PyTorch: an imperative style, high-performance deep learning library},
  author={Paszke, Adam and Gross, Sam and Massa, Francisco and Lerer, Adam and Bradbury, James and Chanan, Gregory and Killeen, Trevor and Lin, Zeming and Gimelshein, Natalia and Antiga, Luca and others},
  booktitle={Proceedings of the 33rd International Conference on Neural Information Processing Systems},
  pages={8026--8037},
  year={2019}
}

@article{missing-reason-1,
  title={Applications of multiple imputation in medical studies: from {AIDS} to {NHANES}},
  author={Barnard, John and Meng, Xiao-Li},
  journal={Statistical Methods in Medical Research},
  volume={8},
  number={1},
  pages={17--36},
  year={1999},
  publisher={Sage Publications Sage CA: Thousand Oaks, CA}
}

@inproceedings{knn,
  title={K-nearest neighbor (k-NN) based missing data imputation},
  author={Pujianto, Utomo and Wibawa, Aji Prasetya and Akbar, Muhammad Iqbal and others},
  booktitle={2019 5th International Conference on Science in Information Technology (ICSITech)},
  pages={83--88},
  year={2019},
  organization={IEEE}
}

@article{softimpute,
  title={Matrix completion and low-rank {SVD} via fast alternating least squares},
  author={Hastie, Trevor and Mazumder, Rahul and Lee, Jason D and Zadeh, Reza},
  journal={The Journal of Machine Learning Research},
  volume={16},
  number={1},
  pages={3367--3402},
  year={2015},
  publisher={JMLR. org}
}

@inproceedings{repaint,
  title={Repaint: Inpainting using denoising diffusion probabilistic models},
  author={Lugmayr, Andreas and Danelljan, Martin and Romero, Andres and Yu, Fisher and Timofte, Radu and Van Gool, Luc},
  booktitle={Proceedings of the IEEE/CVF Conference on Computer Vision and Pattern Recognition},
  pages={11461--11471},
  year={2022}
}

@inproceedings{forest_diff,
  title={Generating and Imputing Tabular Data via Diffusion and Flow-based Gradient-Boosted Trees},
  author={Jolicoeur-Martineau, Alexia and Fatras, Kilian and Kachman, Tal},
  booktitle={International Conference on Artificial Intelligence and Statistics},
  pages={1288--1296},
  year={2024},
  organization={PMLR}
}

@inproceedings{tabddpm,
  title={Tabddpm: Modelling tabular data with diffusion models},
  author={Kotelnikov, Akim and Baranchuk, Dmitry and Rubachev, Ivan and Babenko, Artem},
  booktitle={International Conference on Machine Learning},
  pages={17564--17579},
  year={2023},
  organization={PMLR}
}

@inproceedings{edm,
  title={Elucidating the design space of diffusion-based generative models},
  author={Karras, Tero and Aittala, Miika and Aila, Timo and Laine, Samuli},
  booktitle={Proceedings of the 36th International Conference on Neural Information Processing Systems},
  pages={26565--26577},
  year={2022}
}

@article{em-early,
  title={Maximum likelihood from incomplete data via the {EM} algorithm},
  author={Dempster, Arthur P and Laird, Nan M and Rubin, Donald B},
  journal={Journal of the Royal Statistical Society: Series B (methodological)},
  volume={39},
  number={1},
  pages={1--22},
  year={1977},
  publisher={Wiley Online Library}
}

@article{vae,
  title={Auto-encoding variational bayes},
  author={Kingma, Diederik P and Welling, Max},
  journal={arXiv preprint arXiv:1312.6114},
  year={2013}
}

@inproceedings{gan,
  title={Generative adversarial nets},
  author={Goodfellow, Ian J and Pouget-Abadie, Jean and Mirza, Mehdi and Xu, Bing and Warde-Farley, David and Ozair, Sherjil and Courville, Aaron and Bengio, Yoshua},
  booktitle={Proceedings of the 27th International Conference on Neural Information Processing Systems},
  pages={2672--2680},
  year={2014}
}

@inproceedings{normalizing_flows,
  title={Variational inference with normalizing flows},
  author={Rezende, Danilo and Mohamed, Shakir},
  booktitle={International Conference on Machine Learning},
  pages={1530--1538},
  year={2015},
  organization={PMLR}
}

@inproceedings{ddpm,
  title={Denoising Diffusion Probabilistic Models},
  author={Ho, Jonathan and Jain, Ajay and Abbeel, Pieter},
  booktitle={Proceedings of the 34th International Conference on Neural Information Processing Systems},
  pages={6840--6851},
  year={2020}
}

@inproceedings{tabcsdi,
  title={Diffusion models for missing value imputation in tabular data},
  author={Zheng, Shuhan and Charoenphakdee, Nontawat},
  booktitle={NeurIPS 2022 First Table Representation Workshop},
  year={2022}
}

@inproceedings{mle_diff,
  title={Maximum Likelihood Training of Score-Based Diffusion Models},
  author={Song, Yang and Durkan, Conor and Murray, Iain and Ermon, Stefano},
  booktitle={Advances in Neural Information Processing Systems},
  year={2021}
}

@inproceedings{vp,
  title={Score-Based Generative Modeling through Stochastic Differential Equations},
  author={Song, Yang and Sohl-Dickstein, Jascha and Kingma, Diederik P and Kumar, Abhishek and Ermon, Stefano and Poole, Ben},
  booktitle={The Ninth International Conference on Learning Representations},
  year={2021}
}

@article{em,
  title={Pattern classification with missing data: a review},
  author={Garc{\'\i}a-Laencina, Pedro J and Sancho-G{\'o}mez, Jos{\'e}-Luis and Figueiras-Vidal, An{\'\i}bal R},
  journal={Neural Computing and Applications},
  volume={19},
  pages={263--282},
  year={2010},
  publisher={Springer}
}

@article{mice,
  title={{MICE}: Multivariate imputation by chained equations in {R}},
  author={Van Buuren, Stef and Groothuis-Oudshoorn, Karin},
  journal={Journal of Statistical Software},
  volume={45},
  pages={1--67},
  year={2011}
}

@article{miracle,
  title={Miracle: Causally-aware imputation via learning missing data mechanisms},
  author={Kyono, Trent and Zhang, Yao and Bellot, Alexis and van der Schaar, Mihaela},
  journal={Advances in Neural Information Processing Systems},
  volume={34},
  pages={23806--23817},
  year={2021}
}

@inproceedings{hyperimpute,
  title={Hyperimpute: Generalized iterative imputation with automatic model selection},
  author={Jarrett, Daniel and Cebere, Bogdan C and Liu, Tennison and Curth, Alicia and van der Schaar, Mihaela},
  booktitle={International Conference on Machine Learning},
  pages={9916--9937},
  year={2022},
  organization={PMLR}
}

@inproceedings{gain,
  title={Gain: Missing data imputation using generative adversarial nets},
  author={Yoon, Jinsung and Jordon, James and Schaar, Mihaela},
  booktitle={International Conference on Machine Learning},
  pages={5689--5698},
  year={2018},
  organization={PMLR}
}

@inproceedings{miwae,
  title={{MIWAE}: Deep generative modelling and imputation of incomplete data sets},
  author={Mattei, Pierre-Alexandre and Frellsen, Jes},
  booktitle={International Conference on Machine Learning},
  pages={4413--4423},
  year={2019},
  organization={PMLR}
}

@article{missvae,
  title={Handling incomplete heterogeneous data using vaes},
  author={Nazabal, Alfredo and Olmos, Pablo M and Ghahramani, Zoubin and Valera, Isabel},
  journal={Pattern Recognition},
  volume={107},
  pages={107501},
  year={2020},
  publisher={Elsevier}
}

@inproceedings{adam,
  author    = {Diederik P. Kingma and
               Jimmy Ba},
  title     = {Adam: {A} Method for Stochastic Optimization},
  booktitle = {International Conference on Learning Representations},
  year      = {2015}
}

@article{diffimpute,
  title={{DiffImpute}: Tabular Data Imputation With Denoising Diffusion Probabilistic Model},
  author={Wen, Yizhu and Yi, Kai and Ke, Jing and Shen, Yiqing},
  journal={arXiv preprint arXiv:2403.13863},
  year={2024}
}

@article{missdiff,
  title={{Missdiff}: Training diffusion models on tabular data with missing values},
  author={Ouyang, Yidong and Xie, Liyan and Li, Chongxuan and Cheng, Guang},
  journal={arXiv preprint arXiv:2307.00467},
  year={2023}
}

@inproceedings{mcflow,
  title={Mcflow: Monte carlo flow models for data imputation},
  author={Richardson, Trevor W and Wu, Wencheng and Lin, Lei and Xu, Beilei and Bernal, Edgar A},
  booktitle={Proceedings of the IEEE/CVF Conference on Computer Vision and Pattern Recognition},
  pages={14205--14214},
  year={2020}
}

@article{missforest,
  title={{MissForest}—non-parametric missing value imputation for mixed-type data},
  author={Stekhoven, Daniel J and B{\"u}hlmann, Peter},
  journal={Bioinformatics},
  volume={28},
  number={1},
  pages={112--118},
  year={2012},
  publisher={Oxford University Press}
}

@inproceedings{tdm,
  title={Transformed distribution matching for missing value imputation},
  author={Zhao, He and Sun, Ke and Dezfouli, Amir and Bonilla, Edwin V},
  booktitle={International Conference on Machine Learning},
  pages={42159--42186},
  year={2023},
  organization={PMLR}
}

@inproceedings{mot,
  title={Missing data imputation using optimal transport},
  author={Muzellec, Boris and Josse, Julie and Boyer, Claire and Cuturi, Marco},
  booktitle={International Conference on Machine Learning},
  pages={7130--7140},
  year={2020},
  organization={PMLR}
}

@article{grape,
  title={Handling missing data with graph representation learning},
  author={You, Jiaxuan and Ma, Xiaobai and Ding, Yi and Kochenderfer, Mykel J and Leskovec, Jure},
  journal={Advances in Neural Information Processing Systems},
  volume={33},
  pages={19075--19087},
  year={2020}
}

@inproceedings{igrm,
  title={Data imputation with iterative graph reconstruction},
  author={Zhong, Jiajun and Gui, Ning and Ye, Weiwei},
  booktitle={Proceedings of the AAAI Conference on Artificial Intelligence},
  volume={37},
  number={9},
  pages={11399--11407},
  year={2023}
}

@book{luenberger1997optimization,
  title={Optimization by Vector Space Methods},
  author={Luenberger, David G},
  year={1997},
  publisher={John Wiley \& Sons}
}

@inproceedings{kang2016top,
  title={Top-n recommender system via matrix completion},
  author={Kang, Zhao and Peng, Chong and Cheng, Qiang},
  booktitle={Proceedings of the AAAI Conference on Artificial Intelligence},
  volume={30},
  number={1},
  year={2016}
}

@inproceedings{karras2022elucidating,
  title={Elucidating the design space of diffusion-based generative models},
  author={Karras, Tero and Aittala, Miika and Aila, Timo and Laine, Samuli},
  booktitle={Advances in Neural Information Processing Systems},
  volume={35},
  pages={26565--26577},
  year={2022}
}

@inproceedings{ahamed2024mambatab,
  title={{MambaTab}: A plug-and-play model for learning tabular data},
  author={Ahamed, Md Atik and Cheng, Qiang},
  booktitle={2024 IEEE 7th International Conference on Multimedia Information Processing and Retrieval (MIPR)},
  pages={369--375},
  year={2024},
  organization={IEEE}
}

@article{IsmailFawaz2018deep,
  Title                    = {Deep learning for time series classification: a review},
  Author                   = {Ismail Fawaz, Hassan and Forestier, Germain and Weber, Jonathan and Idoumghar, Lhassane and Muller, Pierre-Alain},
  journal                  = {Data Mining and Knowledge Discovery},
  Year                     = {2019},
  volume                   = {33},
  number                   = {4},
  pages                    = {917--963},
}

@article{ahamed2025tscmamba,
  title={{TSCMamba}: Mamba meets multi-view learning for time series classification},
  author={Ahamed, Md Atik and Cheng, Qiang},
  journal={Information Fusion},
  volume={120},
  pages={103079},
  year={2025},
  publisher={Elsevier}
}

@article{demvsar2006statistical,
  title={Statistical comparisons of classifiers over multiple data sets},
  author={Dem{\v{s}}ar, Janez},
  journal={Journal of Machine Learning Research},
  volume={7},
  number={Jan},
  pages={1--30},
  year={2006}
}

@article{gu2023mamba,
  title={Mamba: Linear-time sequence modeling with selective state spaces},
  author={Gu, Albert and Dao, Tri},
  journal={arXiv preprint arXiv:2312.00752},
  year={2023}
}

@article{scikit-learn,
  title={Scikit-learn: Machine Learning in {P}ython},
  author={Pedregosa, F. and Varoquaux, G. and Gramfort, A. and Michel, V.
          and Thirion, B. and Grisel, O. and Blondel, M. and Prettenhofer, P.
          and Weiss, R. and Dubourg, V. and Vanderplas, J. and Passos, A. and
          Cournapeau, D. and Brucher, M. and Perrot, M. and Duchesnay, E.},
  journal={Journal of Machine Learning Research},
  volume={12},
  pages={2825--2830},
  year={2011}
}

@article{dorogush2018catboost,
  title={CatBoost: gradient boosting with categorical features support},
  author={Dorogush, Anna Veronika and Ershov, Vasily and Gulin, Andrey},
  journal={arXiv preprint arXiv:1810.11363},
  year={2018}
}

@inproceedings{chen2016xgboost,
  title={Xgboost: A scalable tree boosting system},
  author={Chen, Tianqi and Guestrin, Carlos},
  booktitle={Proceedings of the 22nd acm sigkdd international conference on knowledge discovery and data mining},
  pages={785--794},
  year={2016}
}

@inproceedings{song2021scorebased,
  title={Score-Based Generative Modeling through Stochastic Differential Equations},
  author={Yang Song and Jascha Sohl-Dickstein and Diederik P. Kingma and Abhishek Kumar and Stefano Ermon and Ben Poole},
  booktitle={Advances in Neural Information Processing Systems (NeurIPS)},
  year={2021},
  volume={34},
  pages={17144--17158},
  publisher={Curran Associates, Inc.}
}

@article{salimans2022progressive,
  title={Progressive distillation for fast sampling of diffusion models},
  author={Salimans, Tim and Ho, Jonathan},
  journal={arXiv preprint arXiv:2202.00512},
  year={2022}
}

@inproceedings{chen2023sampling,
  title = {Sampling is as Easy as Learning the Score: Theory for Diffusion Models with Minimal Data Assumptions},
  author = {Chen, Sitan and Chewi, Sinho and Li, Jerry and Li, Yuanzhi and Salim, Adil and Zhang, Anru R.},
  booktitle = {Proceedings of the 11th International Conference on Learning Representations (ICLR)},
  year = {2023},
  url = {https://arxiv.org/abs/2209.11215}
}

@article{conforti2023score,
  title = {KL Convergence Guarantees for Score Diffusion Models under Minimal Data Assumptions},
  author = {Conforti, Giovanni and Durmus, Alain and Silveri, Marta G.},
  journal = {SIAM Journal on Mathematics of Data Science},
  volume = {7},
  number = {1},
  pages = {86--109},
  year = {2025},
  url = {https://epubs.siam.org/doi/10.1137/23M1613670}
}

@misc{kelly2023uci,
  title={The UCI machine learning repository},
  author={Kelly, Markelle and Longjohn, Rachel and Nottingham, Kolby},
  year={2023}
}
\appendix
\large\textbf{APPENDIX}
\setcounter{secnumdepth}{1}
\section{Missingness Mechanisms}
\label{app:missingness-mechanisms}
We evaluate imputation performance under three widely adopted missingness mechanisms: MCAR, MAR, and MNAR. These settings determine how the missing entries are distributed within the dataset.

\begin{itemize}
    \item \textbf{Missing Completely At Random (MCAR):} Missing entries are placed uniformly at random, independent of any values in the dataset. This represents the simplest missingness pattern.

    \item \textbf{Missing At Random (MAR):} The pattern of missingness depends on the observed entries but not on the values that are missing. For example, a column may have missing values only when another observed column has a specific range of values.

    \item \textbf{Missing Not At Random (MNAR):} The probability of an entry being missing depends on its own (unobserved) value. This is the most realistic and challenging setting, where missingness is often driven by latent or unrecorded factors.
\end{itemize}

We use the same mask-generation procedures as prior work~\cite{zhang2025diffputer}, and our implementation ensures a consistent evaluation by applying these masks across all methods and datasets. The code for generating these masks is included in the supplementary materials.

\section{Proposition~\ref{thm:warmup} and Proof}
\label{app:standard_properties}
\begin{proposition}
\label{thm:warmup}
Let $Z \in \mathbb{R}^{m \times d}$ be the standardized and binary-encoded data matrix, and $M \in \{0,1\}^{m \times d}$ be the binary mask, where $M_{i,j} = 0$ indicates that the entry $Z_{i,j}$ is observed (non-missing) and $M_{i,j} = 1$ indicates that it is missing. For each feature index $j \in \{1, \dots, d\}$, let $\theta_1^{(j)}: \mathbb{R}^{d-1} \rightarrow \mathbb{R}$ be a predictive function trained using the observed pairs $\mathcal{D}_j := \left\{ (Z_{k, \setminus j}, Z_{k,j}) \mid M_{k,j} = 0, \ \text{for any} \ k \le m \right\}$. Define the imputed matrix $\hat{Z} \in \mathbb{R}^{m \times d}$ entry-wise as:
\[
\hat{Z}_{i,j} =
\begin{cases}
Z_{i,j}, & \text{if } M_{i,j} = 0, \\
\theta_1^{(j)}(Z_{i, \setminus j}), & \text{if } M_{i,j} = 1.
\end{cases}
\]

Then, the following properties hold:
\begin{enumerate}
    \item \textbf{Non-overwriting:} For all $(i,j)$ such that $M_{i,j} = 0$, we have $\hat{Z}_{i,j} = Z_{i,j}$.
    
    \item \textbf{Well-defined mapping:} For all $(i,j)$ such that $M_{i,j} = 1$, the imputed value $\hat{Z}_{i,j}$ is a measurable function of $Z_{i, \setminus j}$ with respect to the model $\theta_1^{(j)}$ trained on $\mathcal{D}_j$.
    
    \item \textbf{One-pass process:} Each column is processed once, and each missing entry is imputed a single time. The full imputation completes in one pass over the feature set.
\end{enumerate}
\end{proposition}

In the following, we will provide a proof of this Proposition. 

\label{app:proof_warmup}
\begin{proof}
Let $Z \in \mathbb{R}^{m \times d}$ and $M \in \{0,1\}^{m \times d}$ be the given input matrix and mask, where $M_{i,j} = 1$ denotes a missing entry with value to be predicted $\hat{Z}_{i,j}$ and $M_{i,j} = 0$ denotes an observed entry with non-missing value $Z_{i,j}$. For each feature index $j \in \{1, \dots, d\}$, let $\theta_1^{(j)}: \mathbb{R}^{d-1} \rightarrow \mathbb{R}$ be a predictive function trained on the dataset:
\[
\mathcal{D}_j = \left\{ (Z_{k, \setminus j}, Z_{k,j}) \mid M_{k,j} = 0 \right\}.
\]
Define $\mathcal{F}j: \mathbb{R}^{d-1} \to \mathbb{R}$ as the function implemented by $\theta_1^{(j)}$. Also, we define the imputed matrix $\hat{Z} \in \mathbb{R}^{m \times d}$ entry-wise by:
\[
\hat{Z}_{ i,j} =
\begin{cases}
Z_{i,j}, & \text{if } M_{i,j} = 0, \\
\theta_1^{(j)}(Z_{i, \setminus j}), & \text{if } M_{i,j} = 1.
\end{cases}
\]

We now prove each of the three properties stated in Proposition~\ref{thm:warmup}.

\textbf{1. Non-overwriting.}  
Let $(i,j)$ be such that $M_{i,j} = 0$. By construction of the imputed matrix, we directly assign:
\[
\hat{Z}_{i,j} = Z_{i,j}.
\]
No transformation, learning, or approximation is applied to observed entries. Hence, $\hat{Z}_{i,j} = Z_{i,j}$ for all non-missing entries, as required.

\textbf{2. Well-defined mapping.}  
Let $(i,j)$ be such that $M_{i,j} = 1$. Then the value $\hat{Z}_{i,j}$ is computed as:
\[
\hat{Z}_{i,j} = \theta_1^{(j)}(Z_{i, \setminus j}),
\]
where $\theta_1^{(j)}$ is trained on $\mathcal{D}_j = \{(Z_{k, \setminus j}, Z_{k,j}) \mid M_{k,j} = 0\}$. We now establish that $\theta_1^{(j)}$ defines a measurable mapping under mild assumptions.

\begin{lemma}
\label{lem:measurable_mapping}
Suppose the training algorithm used to construct $\theta_1^{(j)}$ is either deterministic or conditioned on fixed random seeds. Then, for each $j$, the function $\theta_1^{(j)}$ is measurable, and hence the mapping is well-defined and measurable.
\end{lemma}

\begin{proof}[Proof of Lemma~\ref{lem:measurable_mapping}]
Most standard supervised learning models (e.g., gradient-boosted decision trees, CatBoost, XGBoost) operate by computing a function over finite data $\mathcal{D}_j$, possibly using random subsampling or initialization. If the internal randomness (e.g., seeds, feature sampling) is fixed, then the output model is fully determined by $\mathcal{D}_j$. Therefore, $\theta_1^{(j)}$ is a deterministic function of its inputs, and since it is composed of elementary functions (e.g., trees, weighted sums, thresholds), it is measurable. The composition $\hat{Z}_{i,j} = \theta_1^{(j)}(Z_{i, \setminus j})$ is therefore also measurable.
\end{proof}

This completes the justification for the well-definedness of the imputation for missing entries.

\textbf{3. One-pass process.}  
The imputation algorithm proceeds by iterating once through each feature index $j \in \{1, \dots, d\}$. For each column $j$, a single model $\theta_1^{(j)}$ is trained using only the observed entries in column $j$. Once trained, this model is used to impute missing entries in that column, i.e., those with $M_{i,j} = 1$.

Either in the warm-up or in the polishing stage, no column is revisited after imputation, and no entry is updated more than once. Consequently, the algorithm completes in a single pass over the feature set and produces a complete matrix $\hat{Z}$ with well-defined values for all $(i,j)$.
\end{proof}

\section{Diffusion Details}
\label{app:imputation_diffusion_details}
\subsection{Diffusion Process Used in This Paper}
After the warm-up stage, we obtain $X_a \in \mathbb{R}^{m \times d}$, where missing entries are estimated from the warm-up stage, and observed entries are preserved. Specifically, for all $(i,j)$ such that $M_{i,j} = 0$, we have $(X_a)_{i,j} = X_{i,j}$. The goal of this section is to provide a theoretical foundation for the diffusion-based imputation stage that follows.
For simplicity of notation, we write $X_t$ and $X_0$ as full matrices, but the forward and reverse diffusion processes are applied to each row (i.e., each data sample $x \in \mathbb{R}^d$).
\paragraph{Forward Process.} Starting from $X_0 = X_a$, we define the forward diffusion process as adding Gaussian noise with increasing scale:
\[
X_t = X_0 + \sigma(t) \cdot \varepsilon, \quad \varepsilon \sim \mathcal{N}(0, I),
\]
where $\sigma(t)$ is a monotonically increasing noise schedule, and $X_t$ is the perturbed input at time $t$. The forward distribution is thus:
\[
p(X_t \mid X_0) = \mathcal{N}(X_t; X_0, \sigma^2(t) I).
\]
Let $M$ be the binary mask indicating missing entries. We denote $X_t^{\text{obs}}$ and $X_t^{\text{mis}}$ as the observed and missing components of $X_t$ respectively, based on $M$.

\paragraph{Loss.} 
The denoising network $\theta_2$ is trained to approximate the conditional score function $\nabla_{X_t} \log p(X_t \mid X_0)$, which, under the known Gaussian kernel, has a closed-form expression:
\[
\nabla_{X_t} \log p(X_t \mid X_0) = -\frac{X_t - X_0}{\sigma^2(t)} = -\frac{\varepsilon}{\sigma(t)}.
\]
Thus, the training objective becomes:
\[
\mathcal{L}_{\text{SM}}(\theta_2) = \mathbb{E}_{X_0, \varepsilon, t} \left[ \left\| \theta_2(X_t, t, M) - \nabla_{X_t} \log p(X_t \mid X_0) \right\|_2^2 \right],
\]
which, in practice, reduces to:
\[
\mathcal{L}_{\text{SM}}(\theta_2) = \mathbb{E}_{X_0, \varepsilon, t} \left[ \left\| \theta_2(X_t, t, M) + \frac{\varepsilon}{\sigma(t)} \right\|_2^2 \right].
\]

\paragraph{Reverse Process.}
We now describe the reverse-time stochastic differential equation (SDE) governing the generation of clean data from noise, following the Variance Exploding (VE) framework. Given the forward process defined by
\[
X_t = X_0 + \sigma(t) \cdot \varepsilon, \quad \varepsilon \sim \mathcal{N}(0, I),
\]
the corresponding reverse-time SDE is:
\[
dX_t = -2 \dot{\sigma}(t) \cdot \sigma(t) \cdot \nabla_{X_t} \log p(X_t) \, dt + \sqrt{2 \dot{\sigma}(t) \cdot \sigma(t)} \, dW_t,
\]
where $\sigma(t)$ is the noise scale and $\dot{\sigma}(t)$ denotes its derivative with respect to time. This formulation ensures that as time evolves from $t = T$ to $t = 0$, the process denoises the sample towards the clean data distribution.

In practice, the score function $\nabla_{X_t} \log p(X_t)$ is unknown and is approximated by a denoising network $\theta_2(X_t, t, M)$, trained to estimate the conditional score $\nabla_{X_t} \log p(X_t \mid X_0)$ over the missing entries. The reverse process is discretized as:
\begin{align}
\nonumber X_{t - \Delta t}^{\text{mis}} = &X_t^{\text{mis}} - 2 \dot{\sigma}(t) \cdot \sigma(t) \cdot \theta_2(X_t, t, M) \\
&   + \sqrt{2 \dot{\sigma}(t) \cdot \sigma(t)} \cdot z,\\
X_{t - \Delta t}^{\text{obs}} = &X_a^{\text{obs}},
\end{align}
where the mask $M$ ensures that observed entries remain unchanged at all steps, $z \sim \mathcal{N}(0, I)$, and $\theta_2$ only updates the missing components. The functions $2 \dot{\sigma}(t) \cdot \sigma(t)$ and $\sqrt{2 \dot{\sigma}(t) \cdot \sigma(t)}$ play the roles of $\alpha(t)$ and $\beta(t)$, respectively, and are derived from the continuous noise schedule $\sigma(t)$ used in the forward process.

We now establish that our specific implementation of RefiDiff yields samples from the desired conditional distribution, as stated in Proposition~\ref{thm:diffusion_imputation}.

\begin{proposition}
\label{thm:diffusion_imputation}
Let $X_a \in \mathbb{R}^{m \times d}$ be the warm-up imputed matrix with binary mask $M$, and assume that the denoising model $\theta_2$ approximates the conditional score function $\nabla_{X_t} \log p(X_t \mid X_0)$ for all $t$, where $X_0 = X_a$. Then, the reverse diffusion process with observed values clamped, initialized from $X_T \sim \mathcal{N}(0, \sigma^2(T) I)$ and updated via:
\[
X_{t - \Delta t}^{\text{mis}} = X_t^{\text{mis}} + \alpha(t) \cdot \theta_2(X_t, t, M) + \beta(t) \cdot z_t
\]
yields a sample $\hat{X}_a$ such that:
\[
\hat{X}_a^{\text{obs}} = X_a^{\text{obs}}, \quad \hat{X}_a^{\text{mis}} \sim p(X^{\text{mis}} \mid X^{\text{obs}} = X_a^{\text{obs}})
\]
\end{proposition}

The above result confirms that our non-iterative diffusion-based imputation mechanism, initialized from warm-up filled data $X_a$, and conditioned on the observed entries via mask $M$, is theoretically guaranteed to recover the correct conditional distribution over missing entries. This enables a mathematically principled one-shot denoising approach for missing value imputation.

\subsection{Proof of Proposition~\ref{thm:diffusion_imputation}}
\begin{proof}

The forward process perturbs $X_0$ to obtain:
\[
X_t = X_0 + \sigma(t) \cdot \varepsilon, \quad \varepsilon \sim \mathcal{N}(0, I)
\]
which implies:
\[
p(X_t \mid X_0) = \mathcal{N}(X_t; X_0, \sigma^2(t) I)
\]

Given the forward Gaussian kernel, the conditional score is:
\[
\nabla_{X_t} \log p(X_t \mid X_0) = -\frac{X_t - X_0}{\sigma^2(t)} = -\frac{\varepsilon}{\sigma(t)}
\]
The network $\theta_2$ is trained to match this quantity:
\[
\theta_2(X_t, t, M) \approx \nabla_{X_t} \log p(X_t \mid X_0)
\]

At each timestep $t$, we clamp the observed values:
\[
X_t^{\text{obs}} := X_a^{\text{obs}} \Rightarrow \delta(X_t^{\text{obs}} - X_a^{\text{obs}})
\]
This means:
\[
p(X_t \mid X_a^{\text{obs}}) = \delta(X_t^{\text{obs}} - X_a^{\text{obs}}) \cdot p(X_t^{\text{mis}} \mid X_a^{\text{obs}})
\]

By ~\cite{repaint, vp}, if the forward SDE has known \( p(X_t \mid X_0) \), and the score \( \nabla_{X_t} \log p(X_t \mid X_0) \) is known, then the reverse-time SDE:
\[
dX_t = \left[ -2 \, \sigma(t) \dot{\sigma}(t) \nabla_{X_t} \log p(X_t) \right] dt + \sqrt{2 \sigma(t) \dot{\sigma}(t)} \, dW_t
\]
when integrated from $t = T$ to $t = 0$, samples from $p(X_0)$.

But in our case, $p(X_t)$ is restricted by clamping observed entries, so the reverse SDE becomes:
\[
dX_t^{\text{mis}} = -\nabla_{X_t^{\text{mis}}} \log p(X_t^{\text{mis}} \mid X_a^{\text{obs}}) \, dt + \sqrt{2} \, dW_t
\]

We now discretize time with steps \( t_0 = 0 < t_1 < \dots < t_N = T \), and set:
\[
X_{t_{i-1}}^{\text{mis}} = X_{t_i}^{\text{mis}} + \alpha(t_i) \cdot \theta_2(X_{t_i}, t_i, M) + \beta(t_i) \cdot z_i
\]
This is Langevin dynamics guided by the learned conditional score. Under standard conditions (Lipschitz continuity of score, finite variance of $p(X_t)$, small enough steps $|t_{i} - t_{i-1}|$), it is guaranteed to converge to:
\[
X_0^{\text{mis}} \sim p(X^{\text{mis}} \mid X^{\text{obs}} = X_a^{\text{obs}})
\]
and $X_0^{\text{obs}} = X_a^{\text{obs}}$ by construction. Therefore, from the definition,
\[
\hat{X}_a^{\text{obs}} = X_a^{\text{obs}}, \quad \hat{X}_a^{\text{mis}} \sim p(X^{\text{mis}} \mid X^{\text{obs}} = X_a^{\text{obs}})
\]
\end{proof}

\section{Conditional Consistency of RefiDiff and the Proof}
\label{sec-appendix_cond_consistency}
To further formalize the theoretical behavior of RefiDiff under ideal conditions, we show that it asymptotically recovers the true conditional distribution.
\begin{proposition}[Conditional Consistency of RefiDiff]
\label{thm:cond_consistency}
Let $\mathbf{x} = (\mathbf{x}^{\text{obs}}, \mathbf{x}^{\text{mis}})$ be a data sample with observed and missing components. Suppose RefiDiff is trained with an ideal denoising score function $s_\theta(\mathbf{x}_t, \mathbf{m}) = \nabla_{\mathbf{x}_t} \log p(\mathbf{x}_t \mid \mathbf{m})$, where $\mathbf{m}$ is a binary mask indicating missing entries. Then, 
under ideal conditions (i.e., perfect score function and infinite data), the masked reverse diffusion process asymptotically recovers samples from the true conditional distribution $p(\mathbf{x}^{\text{mis}} \mid \mathbf{x}^{\text{obs}})$, as shown in Proposition~\ref{thm:cond_consistency}:
\[
\mathbf{x}_0^{\text{mis}} \sim p(\mathbf{x}^{\text{mis}} \mid \mathbf{x}^{\text{obs}}).
\]
\end{proposition}

\begin{proof}
The RefiDiff reverse process evolves according to the time-reversed stochastic differential equation (SDE):
\[
d\mathbf{x}_t = \left[f(\mathbf{x}_t, t) - g(t)^2 \nabla_{\mathbf{x}_t} \log p(\mathbf{x}_t \mid \mathbf{m})\right] dt + g(t) d\bar{\mathbf{w}}_t,
\]
where $f$ and $g$ define the forward SDE drift and diffusion coefficients, and $\bar{\mathbf{w}}_t$ denotes standard Brownian motion in reverse time. Under the assumption that $s_\theta(\cdot) = \nabla \log p(\cdot \mid \mathbf{m})$ is exact, and the forward process defines a valid marginal $p(\mathbf{x}_T)$, the result from score-based diffusion theory (e.g., \cite{song2021scorebased}) implies that the reverse process samples from $p(\mathbf{x}_0 \mid \mathbf{m})$.

Because RefiDiff clamps observed values throughout the diffusion trajectory (i.e., $\mathbf{x}_t^{\text{obs}} = \mathbf{x}^{\text{obs}}$), the learned score function and reverse trajectory apply only to $\mathbf{x}^{\text{mis}}$, yielding samples from $p(\mathbf{x}^{\text{mis}} \mid \mathbf{x}^{\text{obs}})$ as desired.
\end{proof}
{\textit{Remark.}} While idealized, the assumption that $s_\theta = \nabla \log p$ (i.e., perfect score function) is standard in diffusion-based generative modeling theory \cite{song2021scorebased} and serves to clarify the limit behavior of RefiDiff.
\section{Theoretical Bound on Approximate Conditional Sampling}
\label{sec-appendix_approx_bd}

We now provide a quantitative approximation guarantee for the conditional sampling behavior of RefiDiff, under mild assumptions on the learned score function and discretization, which is the full version of Proposition~\ref{prop:approximation_bound_brief}.

\begin{proposition}[Approximate Conditional Consistency Bound]
\label{prop:approximation_bound}
Let $\mathbf{x}=(\mathbf{x}^{\rm obs},\mathbf{x}^{\rm mis})$ be a data sample and suppose RefiDiff uses a learned denoiser $s_\theta(\mathbf{x}_t, t, \mathbf{m})$ approximating the true conditional score
$\nabla_{\mathbf{x}_t}\log p(\mathbf{x}_t\mid \mathbf{m})$. Let the reverse diffusion be run in a discretized SDE with timestep $\delta t$, and $N$ independent trajectories are averaged to form an empirical imputed conditional distribution $\hat p(\mathbf{x}^{\rm mis}\mid \mathbf{x}^{\rm obs})$. 
Under the following regularity conditions:
\begin{itemize}
  \item Finite Fisher information and bounded second moments of the data distribution,
  \item Smooth, Lipschitz-continuous score function and noise schedule,
\end{itemize}
if the learned score satisfies the uniform $L^2$ bound
\[
\mathbb{E}_{t\sim [0,T],\,\mathbf{x}_t \sim p_t}\big\|s_\theta(\mathbf{x}_t,t)-\nabla\log p_t(\mathbf{x}_t)\big\|^2 \le \varepsilon_\theta^2,
\]
then
\begin{align*}
  &\mathrm{KL}\bigl(\hat{p}(\mathbf{x}^{\rm mis}\mid \mathbf{x}^{\rm obs})\;\|\;p(\mathbf{x}^{\rm mis}\mid \mathbf{x}^{\rm obs})\bigr)
\;\\
&\le\; C_1\,T\,\varepsilon_\theta^2 + C_2\,\delta t + C_3\,\tfrac{1}{N},  
\end{align*}
where each constant \(C_1,C_2,C_3\) depends polynomially on Fisher information, ambient dimension, noise schedule smoothness, and time horizon \(T\).
\end{proposition}

\begin{proof}[Proof]
First, we provide a sketch of the bound, which follows from three components:
\begin{itemize}
\item \textbf{Score Approximation Error.} From standard score-based generative modeling theory (see \citet{song2021scorebased}), the deviation of the learned score from the true score induces a divergence between the resulting generative distribution and the data distribution. This yields a KL error proportional to $\varepsilon_\theta^2$ under Lipschitz smoothness assumptions on the score.
\item \textbf{Discretization Error.} The continuous-time diffusion process is approximated by a discretized numerical integrator (e.g., Euler–Maruyama or Heun’s method) \cite{salimans2022progressive, karras2022elucidating}, which introduces error of order $O(\delta t)$ per step. When propagated through the denoising sequence, this leads to an accumulated divergence term scaling with $\delta t$.
\item \textbf{Monte Carlo Averaging Error.} Since the final imputation is the average of $N$ stochastic diffusion trajectories, the empirical distribution $\hat{p}$ converges to the expectation over the learned conditional with variance decreasing as $1/N$, similar to 
variance reduction bounds when averaging $N$ Monte Carlo samples. The resulting KL approximation error due to finite averaging follows standard PAC bounds for Monte Carlo estimation.
\end{itemize}
Putting these together yields the desired bound. 

Next, we will provide a detailed step-by-step analysis. As stated above, we decompose the KL divergence between the imputed distribution $\hat{p}(\mathbf{x}^{\text{mis}} \mid \mathbf{x}^{\text{obs}})$ and the true conditional $p(\mathbf{x}^{\text{mis}} \mid \mathbf{x}^{\text{obs}})$ into three contributing sources of error:  1) Score approximation error due to using $s_\theta$ instead of the exact score $\nabla_{\mathbf{x}_t} \log p_t(\mathbf{x}_t)$; 2) Discretization error from solving the reverse SDE with step size $\delta t$; 3) Monte Carlo approximation error from using $N$ diffusion trajectories. We will analyze each source of error below.  

\textbf{(Source 1) Score approximation error.}

Under mild regularity assumptions—including bounded second moments, finite Fisher information relative to the Gaussian perturbation, and sufficiently smooth score functions—the error in the learned score propagates to the KL divergence between the data distributions induced by the approximate and the ideal reverse-time processes.

Specifically, if the learned score function $s_\theta(\mathbf{x}_t, t)$ satisfies the \textbf{$L^2$-accuracy condition}:
\[
\mathbb{E}_{t \sim \mathrm{Uniform}[0,T]}\,\mathbb{E}_{\mathbf{x}_t \sim p_t}\left\| s_\theta(\mathbf{x}_t, t) - \nabla_{\mathbf{x}_t}\log p_t(\mathbf{x}_t) \right\|^2 \le \varepsilon_\theta^2,
\]
where $p_t$ denotes the forward diffusion marginal at time $t$, then the KL divergence between the data distribution induced by the learned reverse SDE, denoted $p_0^{(s_\theta)}$, and the true data distribution $p_0^{\mathrm{true}}$, admits the bound:
\[
\mathrm{KL}\left(p_0^{(s_\theta)} \,\big\|\, p_0^{\mathrm{true}} \right) \leq C_1 \cdot T \cdot \varepsilon_\theta^2,
\]
where $C_1$ is a constant depending on the Fisher information, the noise schedule, and the diffusion horizon $T$. Importantly, $C_1$ scales polynomially (rather than exponentially) with respect to the ambient dimension. This result follows the convergence analysis by \citet{chen2023sampling} on sampling with minimal assumptions, and the KL divergence guarantees of score-based diffusion models under finite Fisher information by \citet{conforti2023score}.

\textbf{(Source 2) Discretization error.}  
The continuous-time reverse SDE is approximated via a numerical solver (e.g., Euler-Maruyama) with step size $\delta t$. Recent theoretical analyses of reverse-time SDE discretizations for diffusion models (see \citet{chen2023sampling, conforti2023score}) show a bound of the form:
\[
\mathrm{KL}(p_{\text{discrete}} \,\|\, p_{\text{cont}}) \leq C_2 \cdot \delta t,
\]
where $p_{\text{discrete}}$ and $p_{\text{cont}}$ denote the distributions induced by the discrete-time numerical solver and the continuous-time reverse process, respectively, and $C_2$ is a constant depending on the smoothness and regularity conditions of the diffusion process.

\textbf{(Source 3) Sampling error.}  
Let $\hat{p}$ be the empirical distribution obtained by averaging over $N$ reverse diffusion samples (trajectories). Standard convergence results for empirical measures (e.g., via the law of large numbers and concentration inequalities like the Dvoretzky–Kiefer–Wolfowitz bound) imply:
\[
\mathrm{KL}(\hat{p} \,\|\, p_{\text{discrete}}) \leq C_3 \cdot \frac{1}{N},
\]
for some constant $C_3$, assuming bounded support or sufficient tail decay.

\textbf{Putting it all together.}  
By triangle inequality for KL (or more carefully, via the chain rule for divergence across approximations), we combine the three sources:
\begin{align*}
  &\mathrm{KL}\bigl(\hat{p}(\mathbf{x}^{\rm mis}\mid \mathbf{x}^{\rm obs})\;\|\;p(\mathbf{x}^{\rm mis}\mid \mathbf{x}^{\rm obs})\bigr)
\;\\
&\le\; C_1\,T\,\varepsilon_\theta^2 + C_2\,\delta t + C_3\,\tfrac{1}{N},  
\end{align*}

This completes the proof.
\end{proof}
\begin{remark}
Proposition~\ref{prop:approximation_bound} shows that RefiDiff's imputation quality improves as the learned score function becomes more accurate, the diffusion process is discretized more finely, and more reverse trajectories are sampled. This provides guidance for tuning the number of steps and ensemble size in practice.
\end{remark}

\section{Datasets Details}
\label{app:datasets}
We use \textbf{nine} public real-world datasets, and they are available at Kaggle\footnote{\url{https://www.kaggle.com/}} or the UCI Machine Learning repository\footnote{\url{https://archive.ics.uci.edu/}}. These include \textbf{five} datasets with only numerical columns (continuous values) and \textbf{four} datasets that contain both numerical and categorical (discrete) columns. The datasets are:

\textbf{California\footnote{\url{https://www.kaggle.com/datasets/camnugent/california-housing-prices}}~\cite{pace1997sparse,geron2022hands}}, which contains 1990 census data on California houses, including longitude, latitude, housing median age, total rooms, etc. 

\textbf{Default~\cite{default}} dataset focuses on predicting customer default payments in Taiwan, comparing six data mining methods for accuracy in estimating default probability. Using a novel Sorting Smoothing Method to estimate true default probability, the study finds that artificial neural networks provide the most accurate forecasting model, with a high coefficient of determination and regression parameters close to ideal values (intercept near zero, slope near one).

\textbf{News~\cite{news}} dataset includes integer and real features, summarizing articles published by Mashable over two years. It is used for classification and regression to predict social media shares (popularity) based on article statistics.

\textbf{Magic~\cite{magic}} dataset consists of Monte Carlo-generated data simulating high-energy gamma particle detection by a ground-based Cherenkov gamma telescope using the imaging technique. It captures Cherenkov radiation from electromagnetic showers initiated by gamma rays in the atmosphere, recorded as pulses on photomultiplier tubes in a camera plane. Shower images, typically elongated clusters, are analyzed using Hillas parameters (from principal component analysis), asymmetry, and other characteristics to distinguish gamma ray signals from cosmic ray background. 

\textbf{Bean\footnote{\url{https://archive.ics.uci.edu/dataset/602/dry+bean+dataset}}~\cite{Koklu2020MulticlassCO}} dataset involves images of seven dry bean varieties, captured using a computer vision system to classify seeds uniformly based on form, shape, type, and structure. After segmentation and feature extraction, 16 features (12 dimensions and 4 shape forms) were derived from the bean grains for classification.

\textbf{Gesture~\cite{gesture}} dataset consists of features from seven videos of people gesticulating, designed for studying Gesture Phase Segmentation. Each video has two files: a raw file with frame-by-frame positions of hands, wrists, head, and spine, and a processed file with velocity and acceleration data for hands and wrists. 

\textbf{Letter~\cite{letter}}, where the goal is to classify black-and-white rectangular pixel images as one of the 26 English capital letters. The dataset includes 20,000 unique images from 20 fonts, with each letter randomly distorted. Images are represented by 16 numerical attributes (statistical moments and edge counts). 

\textbf{Adult~\cite{adult}} dataset, extracted from the 1994 Census by Barry Becker, contains  14 categorical and integer features. It is used for classification to predict whether an individual's annual income exceeds $\$50,000$ based on clean records meeting specific conditions (e.g., age $> 16$, hours worked $> 0$). 

\textbf{Shoppers~\cite{shoppers}} dataset contains 12,330 user sessions over a year, with 17 integer and real features. It is used for classification and clustering to predict purchasing intent, with 84.5\% (10,422) sessions not resulting in a purchase and 15.5\% (1,908) ending in a purchase.

The statistics of these datasets are presented in Table~\ref{tbl:stat-dataset}. 
\begin{table*}
    \centering
	\begin{tabular}{lccccccccc}
            \toprule
            \textbf{Dataset}  &  \# Rows  & \# Num & \# Cat & \# In-sample & \# Out-of-sample  \\
            \midrule 
            \textbf{California} & $20,433$ & $9$ & - & $14,303$ & $6,130$   \\
             \textbf{Default} & $30,000$ & $14$ & $9$ & $21,000$ & $9,000$\\
             \textbf{News} & $39,644$ & $45$ & $2$ & $27,750$ & $11,894$  \\
             \textbf{Magic} & $19,020$ & $10$ & - & $13,314$ & $5,706$  \\
             \textbf{Bean} & $13,610$ & $16$ & - & $9,527$ & $4,083$ \\
             \textbf{Gesture} & $9,522$ & $32$ & - & $6,665$ & $2,857$ \\
            \textbf{Letter} & $20,000$ & $16$ & - & $14,000$ & $6,000$ \\
            \textbf{Adult} & $32,561$ & $6$ & $8$ & $22,792$ & $9,769$ \\
            \textbf{Shoppers} & $12,330$ & $10$ & $7$ & $8,631$ & $3,699$ \\
            
		\bottomrule
		\end{tabular}
        \caption{Dataset statistics. \# Num stands for the number of numerical columns, and \# Cat stands for the number of categorical columns. "-" represents the corresponding entry is not present.} 
        \label{tbl:stat-dataset}

\end{table*}

\section{Implementation Details and Hyperparameters}
\label{app:implementation_details}
This section outlines the implementation environment, baseline configurations, and key hyperparameters used in our proposed method.

\paragraph{Environment.} 
Our framework is implemented using Python and PyTorch~\cite{pytorch}, along with supporting libraries such as scikit-learn~\citep{scikit-learn}. Experiments were conducted on a cluster running Ubuntu, with each compute node equipped with 4 NVIDIA V100 GPUs with 32-core CPUs and 128GB RAM. However, our codebase can be executed on a single GPU with 32GB VRAM, alongside 8-core CPUs and 32GB RAM. We used the same environment to run the baseline methods.

\paragraph{Baselines.}
We closely follow the official implementations of all baseline methods and adopt the recommended hyperparameters for fair comparison:
\begin{itemize}
    \item \textbf{DIFFPUTER:} We used the official implementation, which was available with the ICLR paper at \url{https://github.com/hengruizhang98/DiffPuter}.
    
    \item \textbf{ReMasker:} The official GitHub repository is used along with the recommended settings: \url{https://github.com/tydusky/remasker}.
    
    \item \textbf{KNN:} We use the KNN imputer from scikit-learn, following the common practice from DIFFPUTER, where the number of neighbors is set to $\sqrt{m}$, with $m$ being the number of samples.
    
    \item \textbf{HyperImpute and Others:} For HyperImpute~\cite{hyperimpute}, MissForest~\cite{missforest}, MICE~\cite{mice}, SoftImpute~\cite{softimpute}, EM~\cite{em}, GAIN~\cite{gain}, MIRACLE~\cite{miracle}, and MIWAE~\cite{miwae}, we used the implementations provided in the HyperImpute package~\footnote{\url{https://github.com/vanderschaarlab/hyperimpute}}. All hyperparameter tuning follows the same settings reported in DIFFPUTER's Appendix D.6~\cite{zhang2025diffputer}.
\end{itemize}
\paragraph{Ours.} 
To implement our framework, we adopt \texttt{XGBoost}~\cite{chen2016xgboost} as the default regressor with standard parameters, including 100 estimators and a learning rate of 0.1. For categorical variables, we employ \texttt{CatBoost}~\cite{dorogush2018catboost} as the default classifier with its default configuration. Our framework is modular and supports alternative predictive models such as Random Forest, ReMasker, and Support Vector Classifier (SVC), enabling flexibility based on user preference or domain requirements.

For the denoising network $\theta_2$, we use a custom lightweight Mamba-based architecture with a hidden dimension $hd=32$, comprising two upsampling and two downsampling blocks. In the upsampling blocks, we consider dimensionality increase in the form of $32\to 64\to 128$, and then in the downsampling blocks, we consider a gradual decrease in the form of $128\to 64\to 32$.

We set the number of reverse diffusion trials to $N=10$ by default, following the strategy used in DIFFPUTER. However, as demonstrated in Figure~\ref{subfig:california_N_samples} and Figure~\ref{subfig:magic_N_samples}, our method maintains stable performance even with smaller values of $N$ (e.g., $N=1$ or $N=3$), demonstrating robustness to stochasticity. The number of sampling steps during the reverse diffusion process is fixed at 50, consistent with DIFFPUTER.

\section{Statistical Analysis}
\label{app:statistical_analysis}
We present critical difference (CD) diagrams inspired by~\cite{demvsar2006statistical,ahamed2025tscmamba,IsmailFawaz2018deep} in Figure~\ref{subfig:num_cd} and Figure~\ref{subfig:cat_cd} to compare the average ranks of imputation methods for each setting and metric (e.g., MNAR In-MAE). For regression tasks involving numerical columns, our method achieves the top rank and is statistically distinguishable from all other baselines, demonstrating clear superiority. For categorical imputation, our method also ranks highest and forms a statistically indistinguishable group with several strong baselines, such as ReMasker and DIFFPUTER. This highlights the robustness of our approach across both numerical and categorical modalities.
\begin{figure*}
    \centering
    
    \begin{subfigure}[t]{\linewidth}
    
        \centering        
        \includegraphics[width=\linewidth]{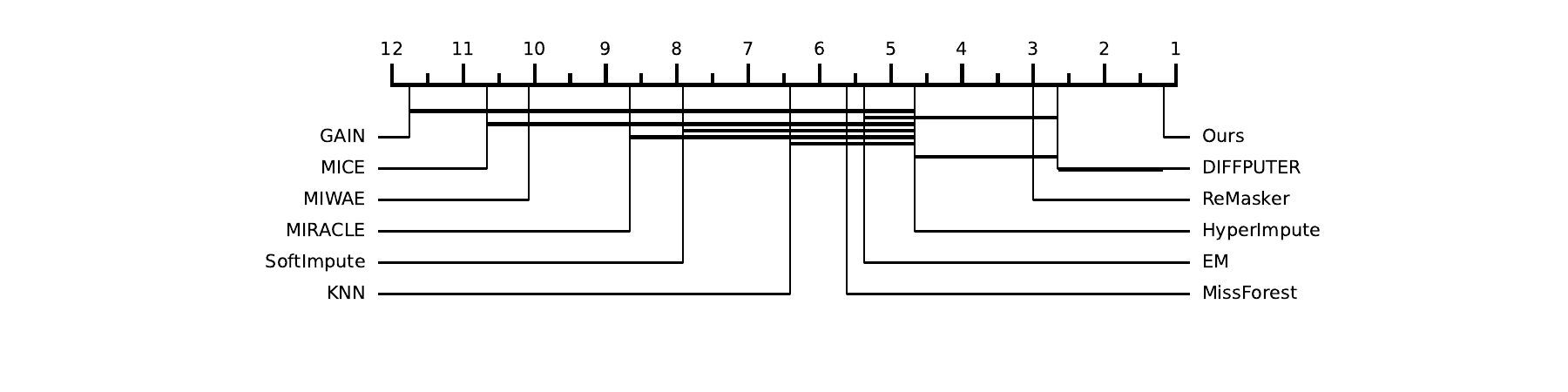}
        \caption{Numerical imputation}
        \label{subfig:num_cd}
    \end{subfigure}
    \hfill
    \begin{subfigure}[t]{\linewidth}
        \centering
        \includegraphics[width=\linewidth]{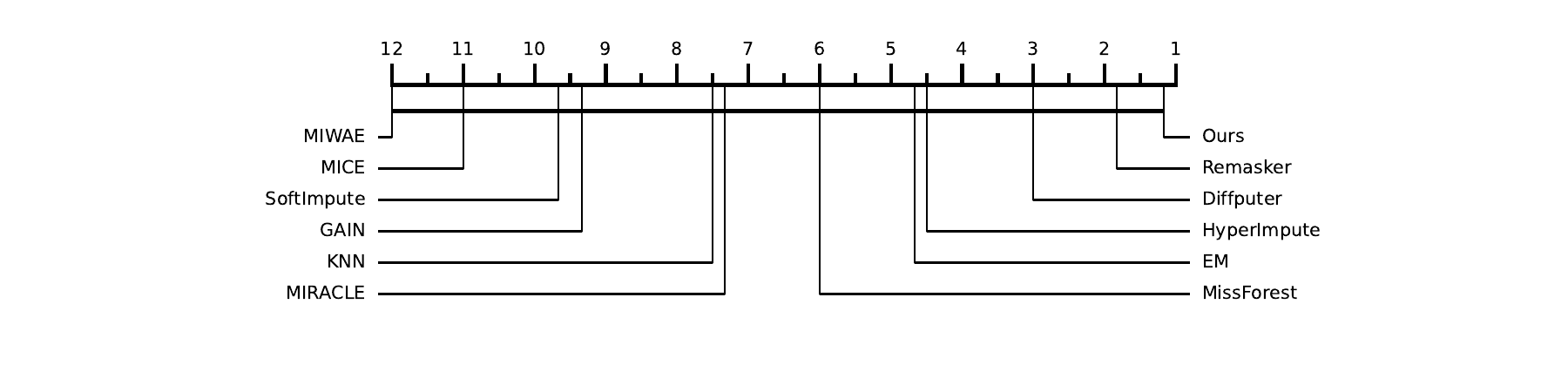}
        \caption{Categorical imputation}
        \label{subfig:cat_cd}
    \end{subfigure}

    \caption{Critical difference diagram for numerical and categorical columns performance metrics.}
    \label{fig:cd_diagram}
    
\end{figure*}

\section{Variation analysis}
\label{app:variation}
To understand the stability of different imputation methods, we analyze their performance variation under ten random masks for the top performing methods. Here we demonstrate them by mean and standard deviation in Figure~\ref{fig:mcar_in_MAE} and Figure~\ref{fig:mcar_in_RMSE} employing the MCAR in-sample setting. Our method consistently demonstrates lower standard deviation across datasets, indicating more stable behavior despite the randomness of the missing values. In contrast, methods like KNN and Remasker show higher variance on datasets such as adult and magic, suggesting a lack of consistency when the missingness pattern changes. This analysis highlights that our approach is less sensitive to variations in random masking, making it more reliable across repeated runs.
\begin{figure*}
    \centering
    \includegraphics[width=\linewidth]{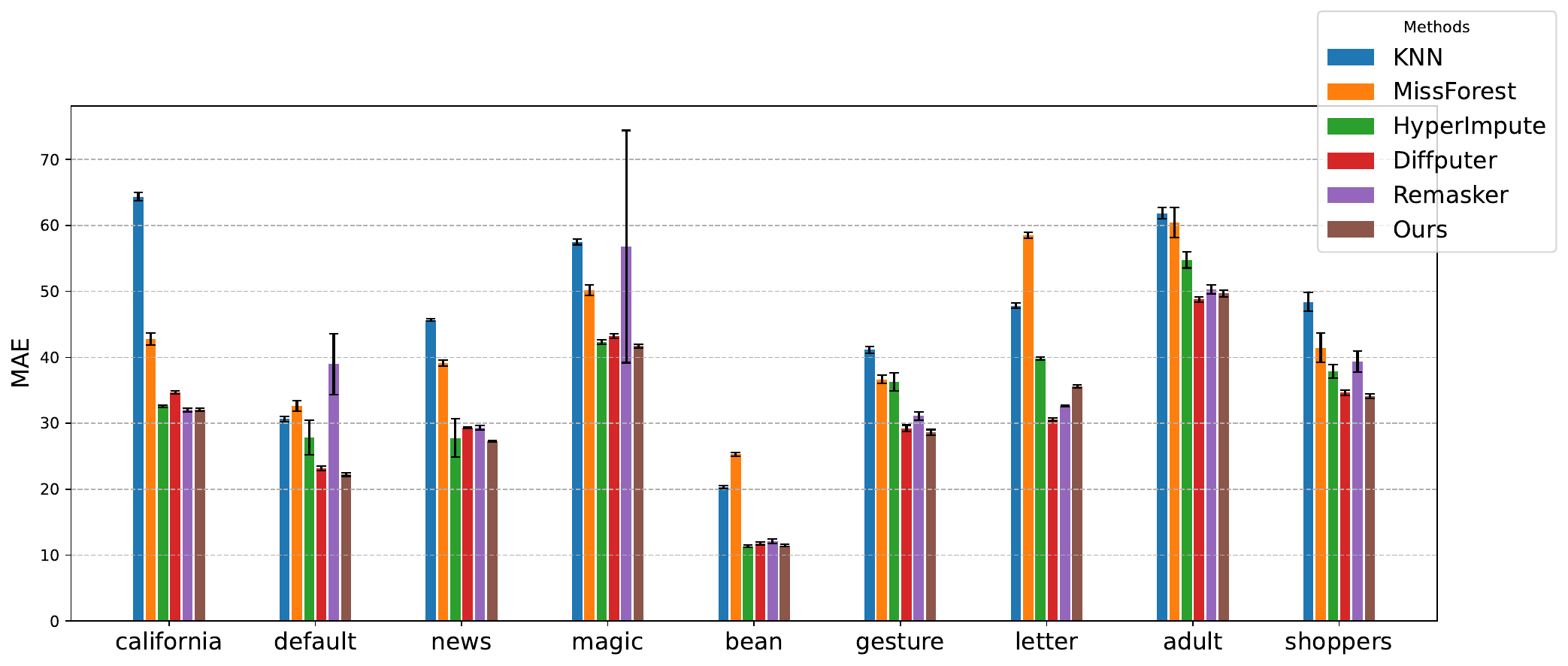}
    \caption{Variation analysis for MCAR In-sample MAE.}
    \label{fig:mcar_in_MAE}
\end{figure*}
\begin{figure*}
    \centering
    \includegraphics[width=\linewidth]{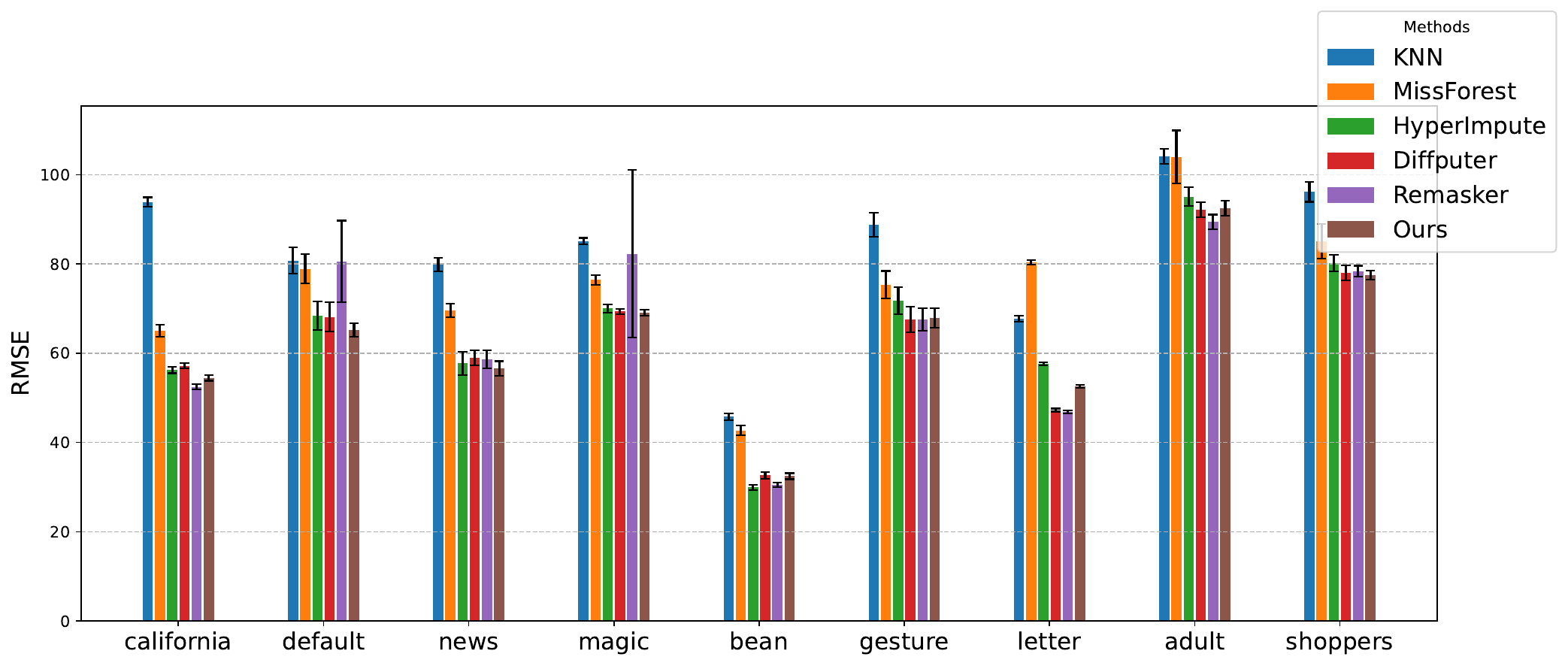}
    \caption{Variation analysis for MCAR In-sample RMSE.}
    \label{fig:mcar_in_RMSE}
\end{figure*}

\section{Additional Ablation Studies}
\label{app:additional_ablation}
\paragraph{Robustness of Our Designed Denoiser.} 
To evaluate the robustness of our designed denoiser and general applicability, we examine whether our denoiser remains effective outside our proposed framework. Specifically, we replace DIFFPUTER's original denoising network with our memory-efficient $\theta_2$ while keeping all other components of DIFFPUTER unchanged. As shown in Figure~\ref{fig:robustness_of_our_denoiser}, our denoiser successfully retains comparable performance in the DIFFPUTER framework under the MNAR setting for both the California and Magic datasets. Notably, on the Magic dataset, it even outperforms DIFFPUTER's original model despite having far fewer parameters. These findings demonstrate that our denoiser is not only memory-efficient but also robust, maintaining strong performance even when deployed within external imputation pipelines.
\begin{figure}[H]
    \centering
    \includegraphics[width=\linewidth]{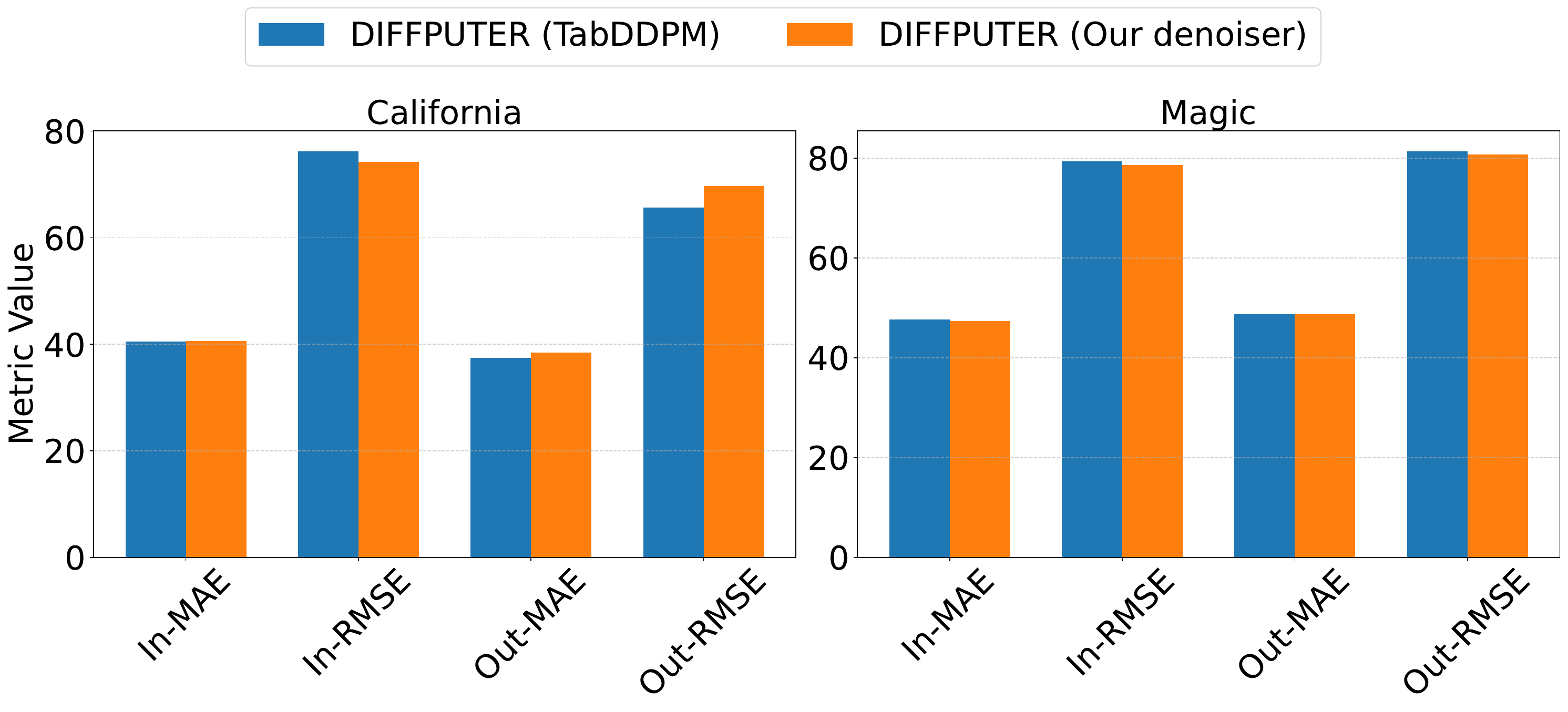}
    \caption{Robustness analysis of our designed denoiser.}
    \label{fig:robustness_of_our_denoiser}
\end{figure}
\label{denoiser-in-DiffPutter}

\paragraph{Experiment on Hidden Dimension.}

We perform an ablation study to assess the effect of varying the hidden dimension $hd$ of our denoising network $\theta_2$ under the MNAR setting using the California dataset with a single random mask. As shown in Figure~\ref{fig:ablation_hd}, increasing $hd$ generally improves both in-sample and out-of-sample performance up to a point. While higher values of $hd$ introduce greater modeling capacity, they also increase the number of parameters and computational cost. Our results demonstrate that setting $hd = 32$ offers a favorable balance, delivering competitive performance while maintaining a lightweight footprint. This makes it an effective default choice for achieving both efficiency and accuracy in imputation.
\begin{figure}[H]
    \centering
    \begin{subfigure}[t]{0.48\linewidth}
    
        \centering        
        \includegraphics[width=\linewidth]{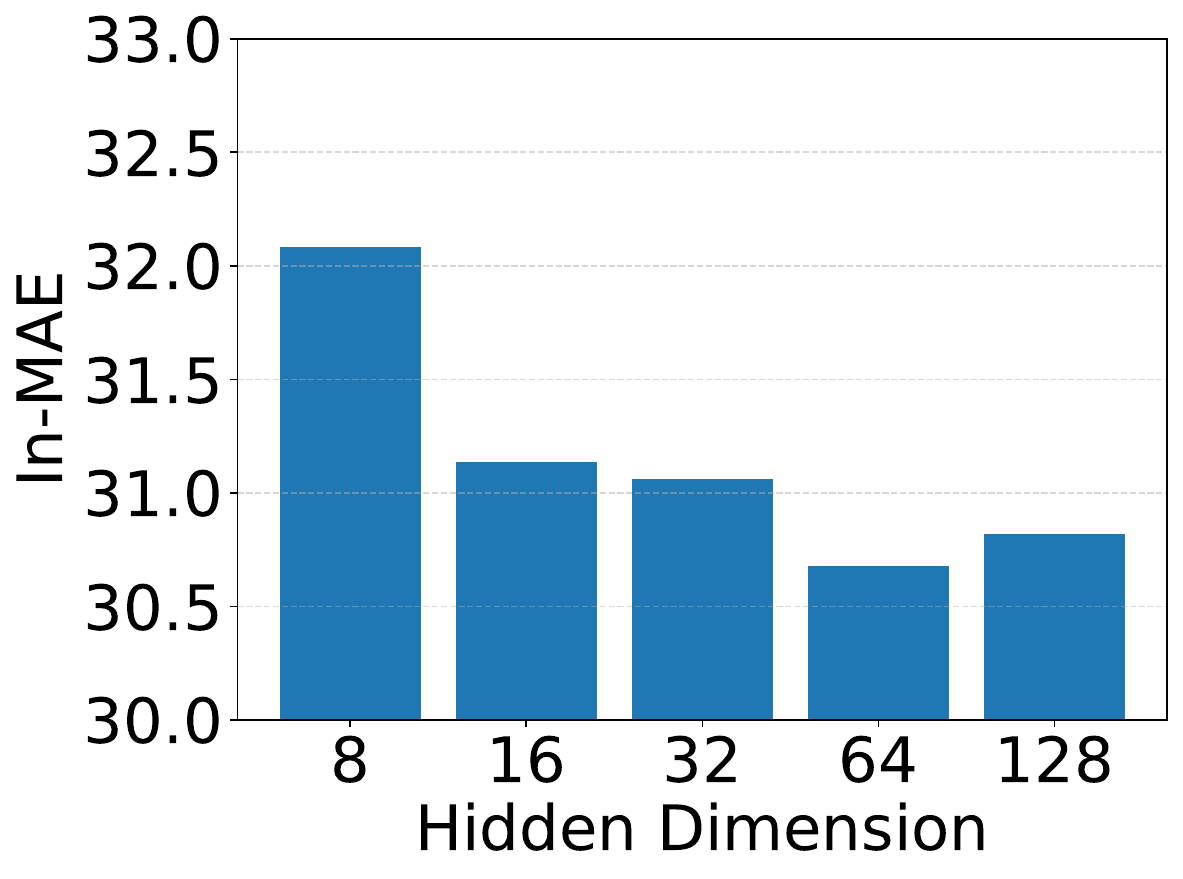}
        \caption{In-sample MAE}
        \label{subfig:in_mae_vs_hd}
    \end{subfigure}
    \hfill
    \begin{subfigure}[t]{0.48\linewidth}
        \centering
        \includegraphics[width=\linewidth]{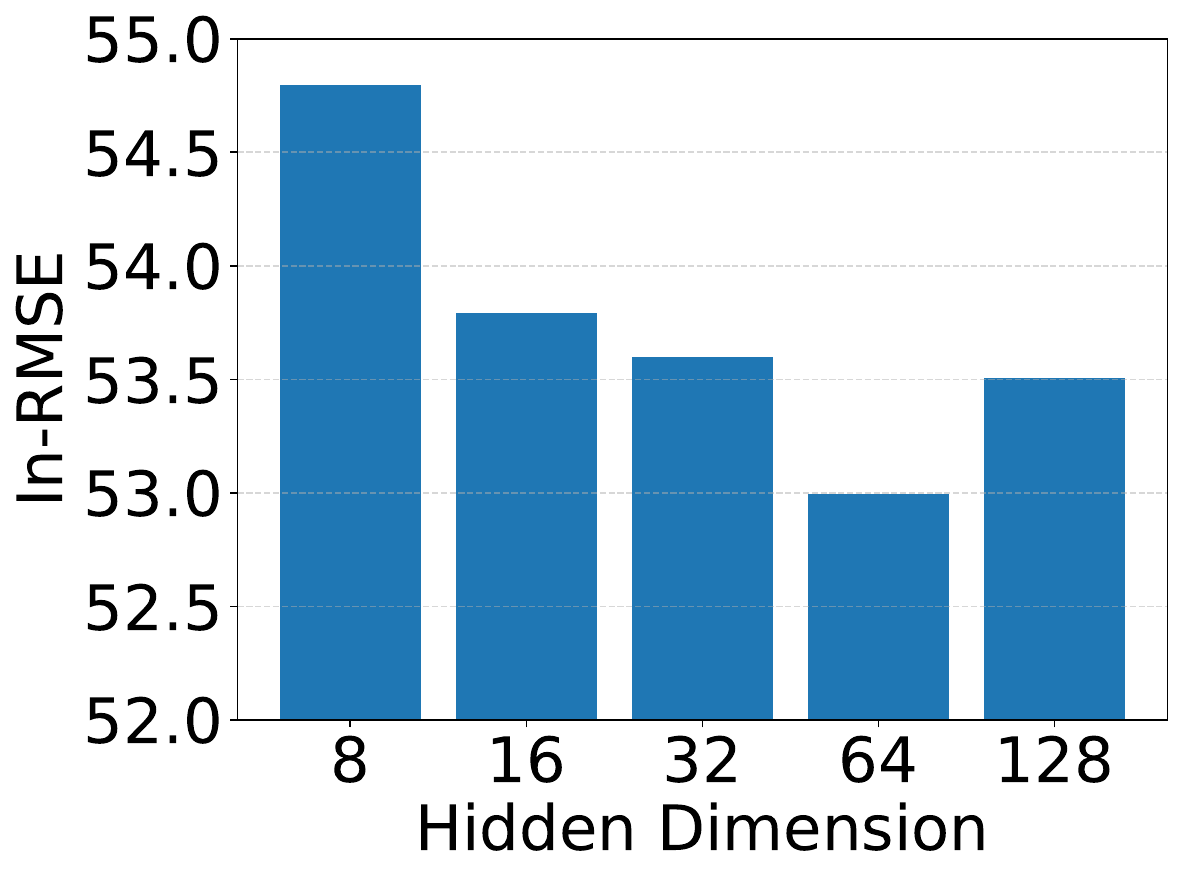}
        \caption{In-sample RMSE}
        \label{subfig:in_rmse_vs_hd}
    \end{subfigure}
    \hfill
    \begin{subfigure}[t]{0.48\linewidth}
        \centering
        \includegraphics[width=\linewidth]{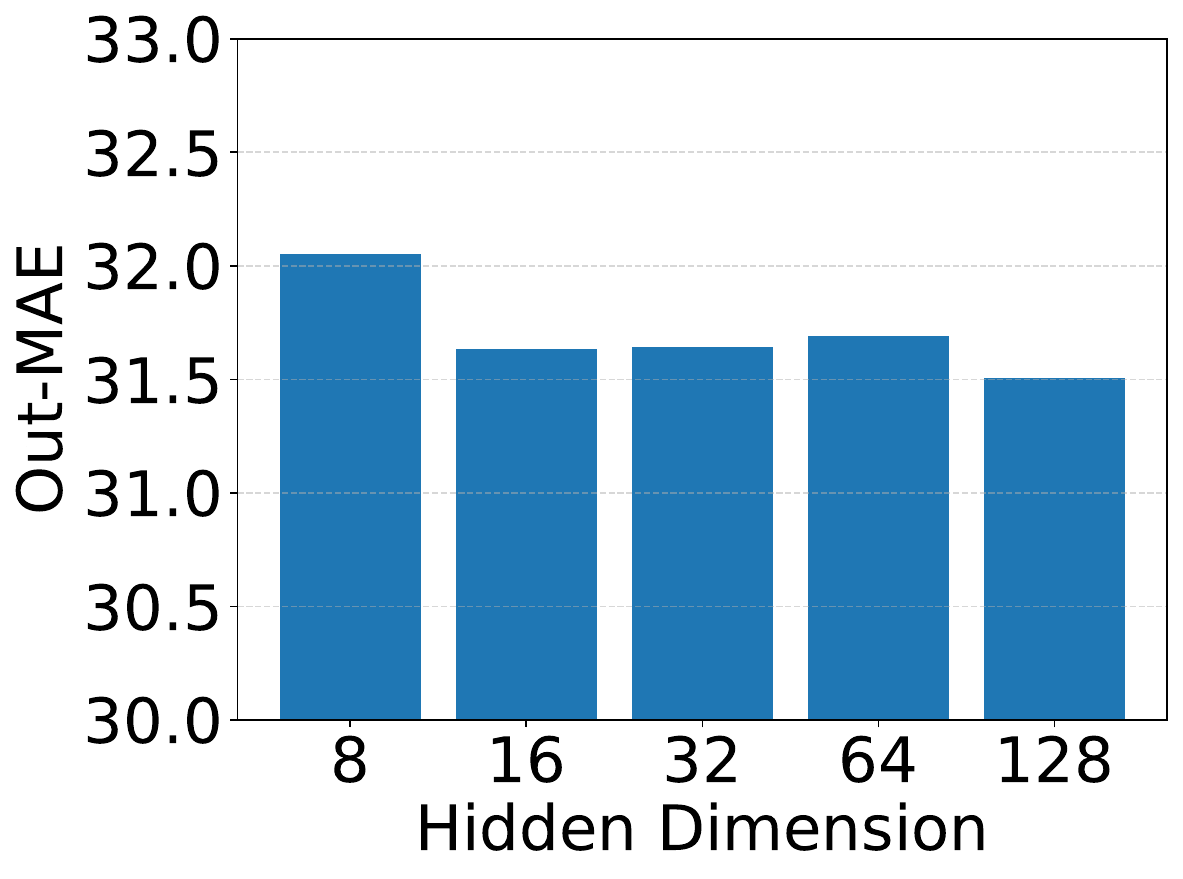}
        \caption{Out-of-sample MAE}
        \label{subfig:out_mae_vs_hd}
        
    \end{subfigure}
    \hfill
    \begin{subfigure}[t]{0.48\linewidth}
        \centering
        \includegraphics[width=\linewidth]{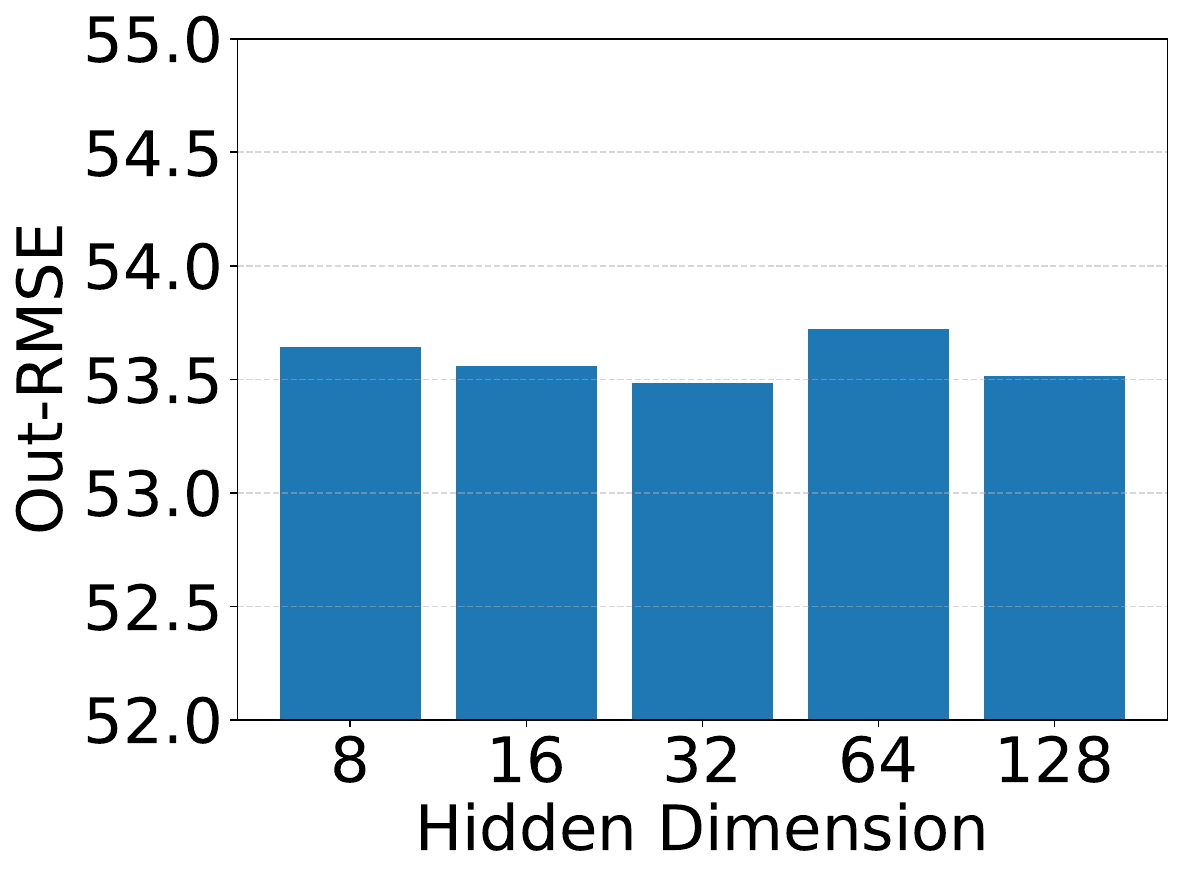}
        \caption{Out-of-sample RMSE}
        \label{subfig:out_rmse_vs_hd}
    
    \end{subfigure}
    \caption{Ablation study on the hidden dimension $hd$ of the denoising network $\theta_2$ under the MNAR setting on the California dataset.}
    \label{fig:ablation_hd}
    
\end{figure}

\paragraph{Long Range Dependency and Selectivity.}

Mamba leverages a selective state-space formulation that dynamically adapts state transitions based on input content, allowing it to prioritize relevant information while suppressing noise. This dynamic and content-aware mechanism addresses key limitations in earlier models, such as convolutional networks and Linear Time-Invariant (LTI) state-space models, which process inputs uniformly without regard to contextual relevance. To empirically validate this selective capability, we adopted the selective copying task, a synthetic benchmark designed to assess a model’s ability to recall contextually important tokens embedded among distractors. The task involves memorizing specific tokens (denoted by numbers in Figure~\ref{fig:slcopy}) within long sequences of irrelevant noise tokens (shown in zeros), where the positions of the relevant tokens vary between samples, thus requiring content-sensitive recall. We followed the evaluation protocol outlined in the original Mamba paper, training models on sequences of length 4096 drawn from a 16-token vocabulary. The objective was to accurately recover 16 target tokens interspersed throughout the sequence. Experimental results show that Mamba achieves impressive performance, attaining an accuracy of \textbf{98.73\%} with two layers and \textbf{99.41\%} with three layers. These findings highlight Mamba’s strong capacity for long-range dependency modeling and selective attention, affirming its suitability for tasks that demand precise memory and effective noise suppression.
\begin{figure}[H]
    \centering
    \includegraphics[width=\linewidth]{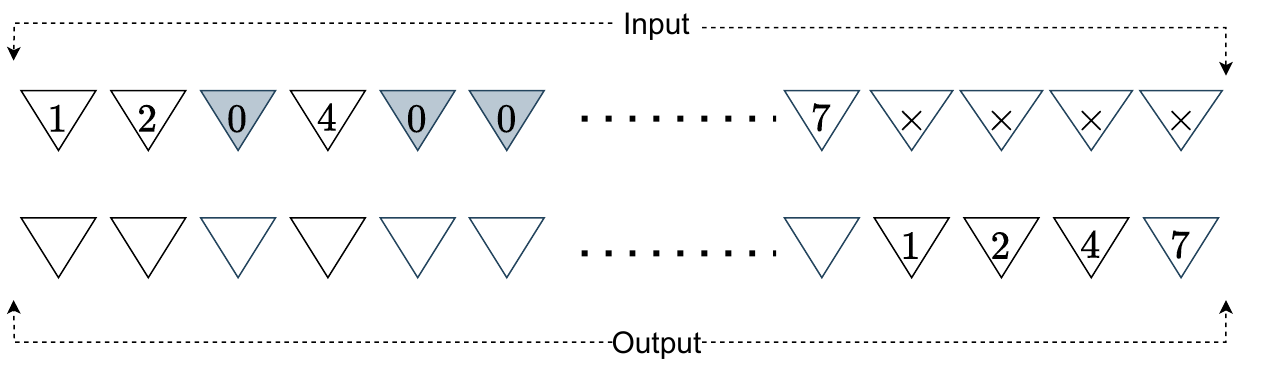}
    \caption{Conceptual diagram of long-range dependency and selectivity.}
    \label{fig:slcopy}
\end{figure}

\paragraph{Sensitivity Analysis.}

We evaluate RefiDiff's sensitivity with respect to two primary factors: (1) the choice of denoising network and (2) the number of sampling trials $N$ in the reverse diffusion process.

First, we compare our Mamba-based denoising module \( \theta_2 \) against DIFFPUTER's larger Transformer-based denoiser (TabDDPM~\cite{tabddpm}). As shown in Figures~\ref{subfig:california_denoiser} and~\ref{subfig:magic_denoiser}, our framework maintains comparable MAE and RMSE scores across datasets (within 2\% difference), demonstrating robustness to denoiser architecture changes. Notably, our denoiser achieves these results with significantly fewer parameters (Figure~\ref{subfig:denoising_param}), offering substantial computational efficiency benefits.

Second, we examine the effect of varying the number of diffusion sampling trials $N$. While DIFFPUTER reported $N$ to be critical for reducing stochastic variance, our method remains remarkably stable even with minimal sampling. As shown in Figures~\ref{subfig:california_N_samples} and~\ref{subfig:magic_N_samples}, our approach maintains consistent performance with as few as 1 or 3 trials, whereas DIFFPUTER exhibits performance drops of up to 20\% with reduced sampling. This stability eliminates the need for computationally expensive ensembling during inference.

\begin{figure}[H]
    \centering
    
    \begin{subfigure}[t]{0.4\linewidth}
    
        \centering        
        \includegraphics[width=\linewidth]{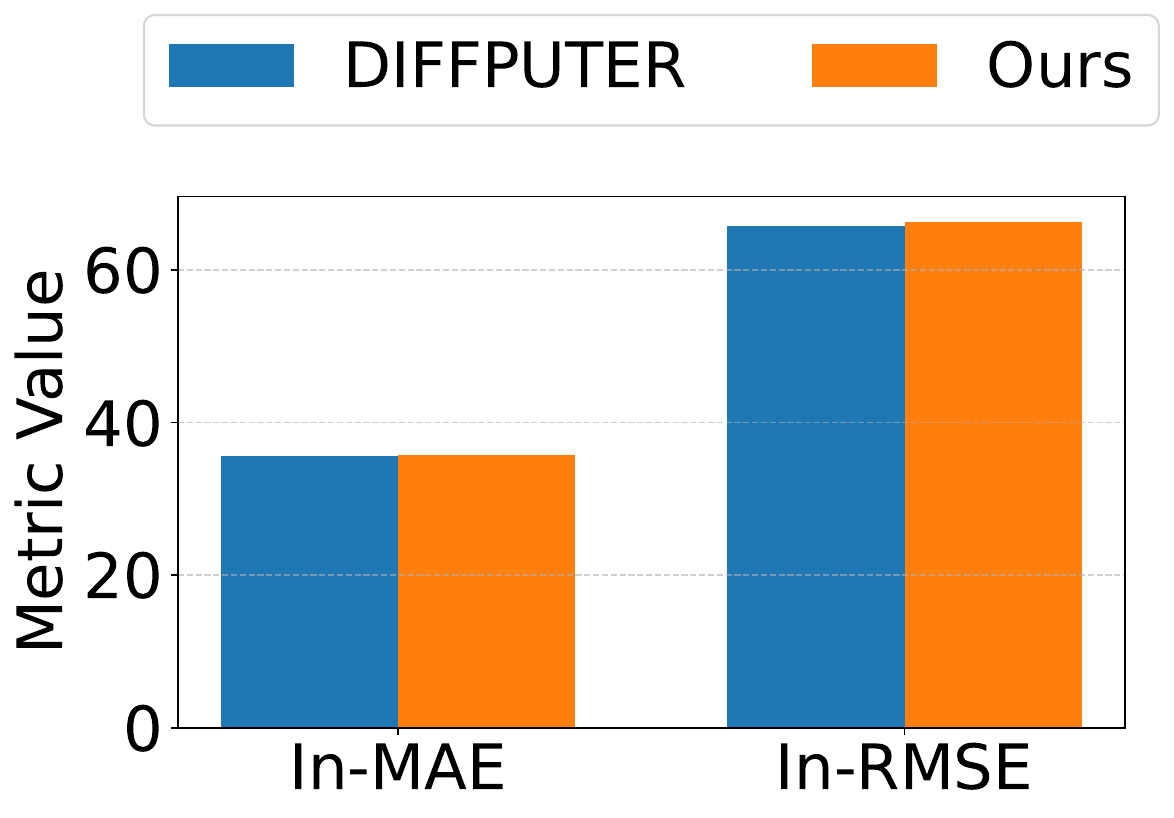}
        \caption{California}
        \label{subfig:california_denoiser}
    \end{subfigure}
    \hfill
    \begin{subfigure}[t]{0.4\linewidth}
        \centering
        \includegraphics[width=\linewidth]{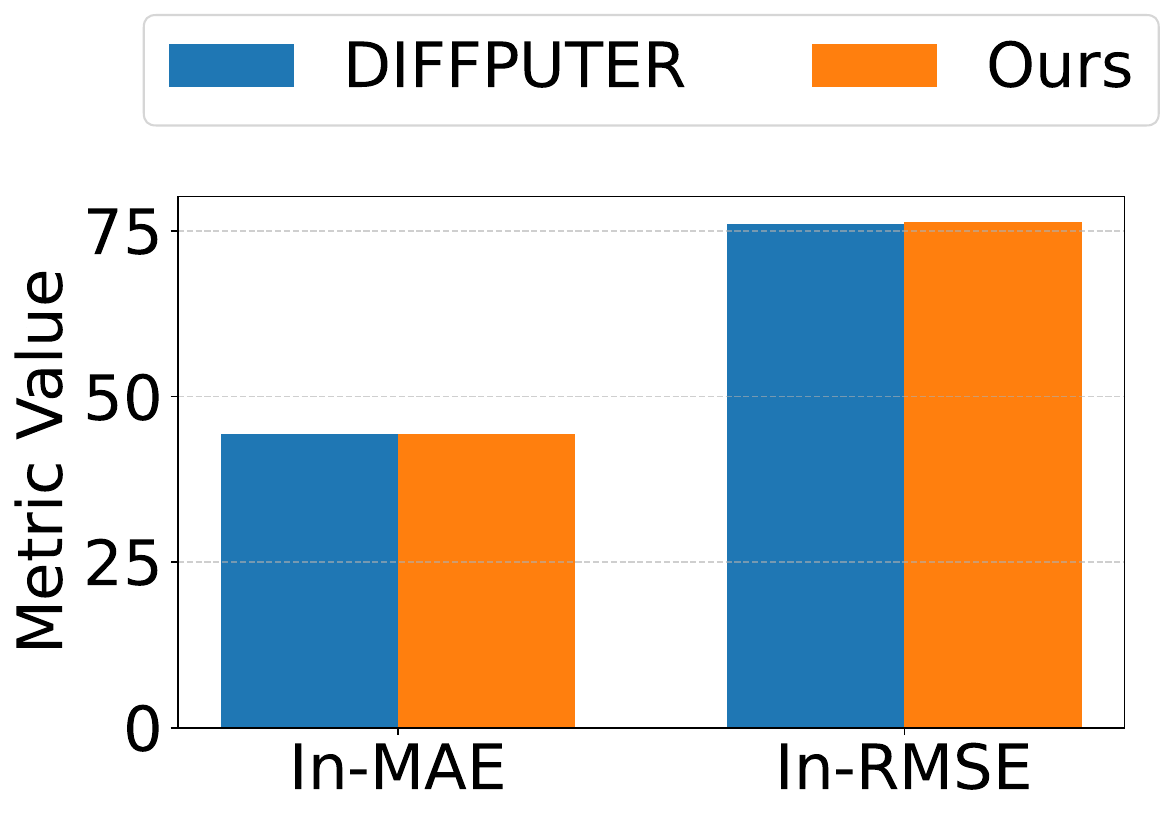}
        \caption{Magic}
        \label{subfig:magic_denoiser}
    \end{subfigure}
    \hfill
    \begin{subfigure}[t]{0.4\linewidth}
        \centering
        \includegraphics[width=\linewidth]{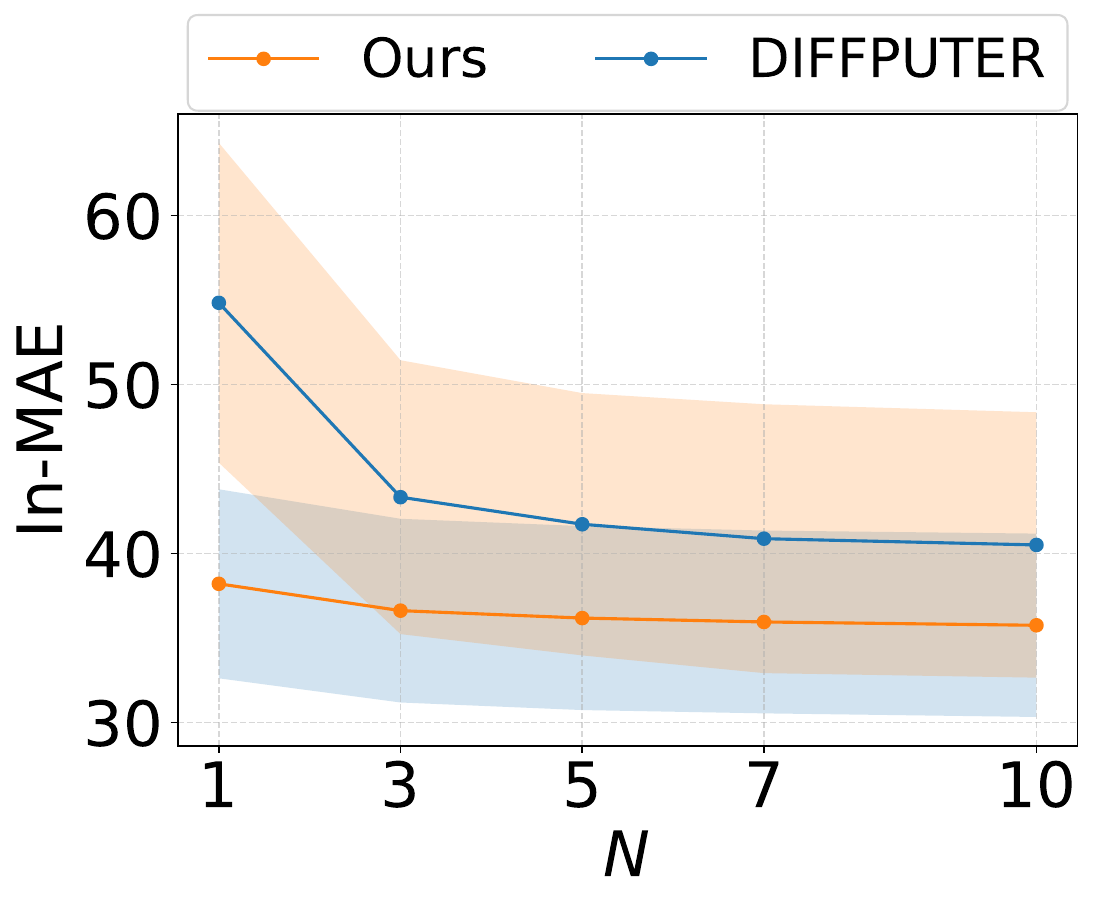}
        \caption{California}
        \label{subfig:california_N_samples}
        
    \end{subfigure}
    \hfill
    \begin{subfigure}[t]{0.4\linewidth}
        \centering
        \includegraphics[width=\linewidth]{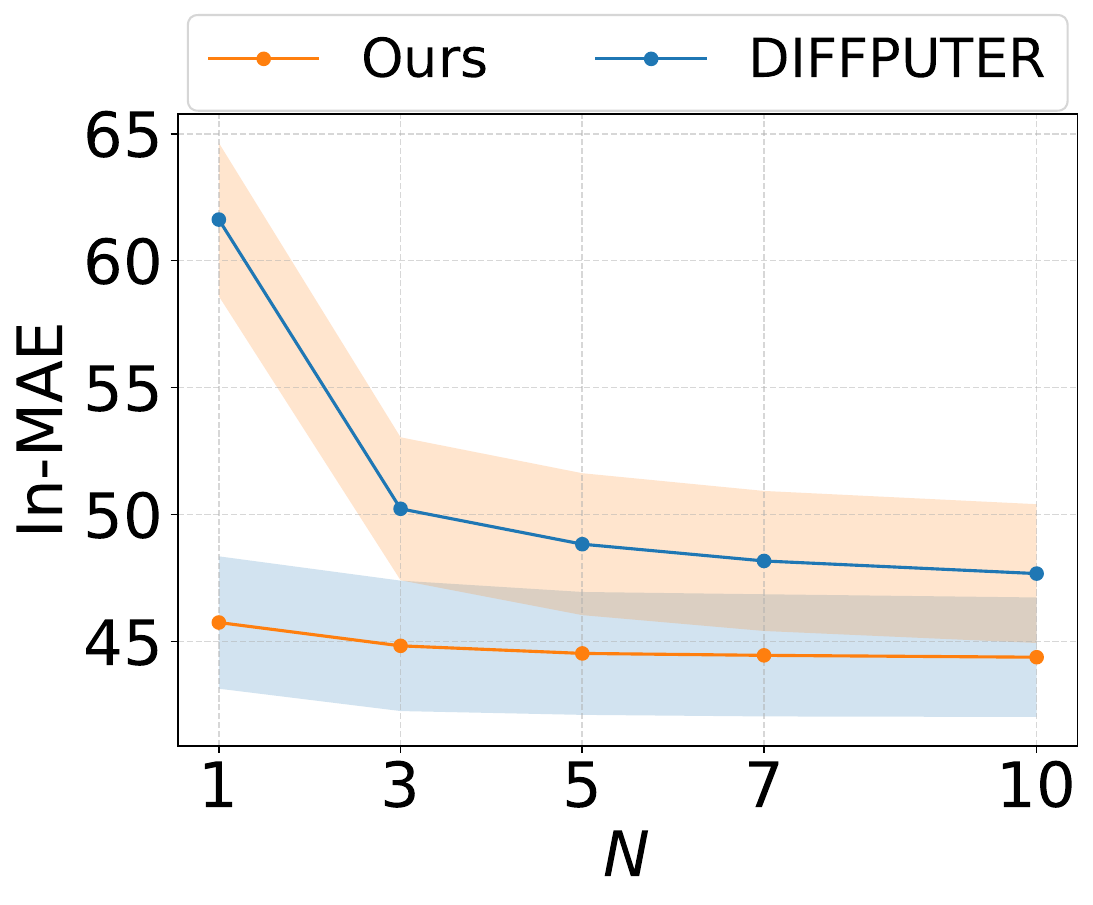}
        \caption{Magic}
        \label{subfig:magic_N_samples}
    
    \end{subfigure}
    \caption{Sensitivity analysis of our framework. Subfigures (a, b) evaluate denoising networks, while (c, d) assess the impact of sampling trials $N$ during reverse diffusion.}
    \label{fig:sensitivity}
\end{figure}

\section{Detailed results}
\label{app:detailed_result}
\begin{table*}
\centering

\begin{tabular}{llccccccccc}
\toprule
& \textbf{Method} & \textbf{California} & \textbf{Default} &\textbf{News} &\textbf{Magic} & \textbf{Bean} & \textbf{Gesture} & \textbf{Letter} & \textbf{Adult} & \textbf{Shoppers}  \\
\midrule

\multirow{13}{*}{\rotatebox[origin=c]{90}{\textbf{MAE}}} 
& EM & 40.55 & 32.88 & 39.89 & 53.52 & 11.10 & 39.37 & 56.70 & 62.72 & 45.03 \\
& MIWAE & 84.00 & 59.73 & 66.00 & 83.29 & 64.84 & 63.84 & 73.87 & 64.84 & 53.49 \\
& GAIN & 83.13 & 89.84 & 59.80 & 67.32 & 47.68 & 81.16 & 68.23 & 130.16 & 55.22 \\
& SoftImpute & 61.55 & 51.66 & 65.01 & 65.32 & 35.74 & 59.33 & 66.59 & 70.21 & 62.96 \\
& MICE & 58.94 & 63.36 & 63.96 & 79.20 & 20.91 & 69.27 & 81.79 & 103.90 & 79.80 \\
& MIRACLE & 45.72 & 44.47 & 41.07 & 50.99 & 19.52 & 76.46 & 70.75 & 64.54 & 77.29 \\
& KNN & 79.31 & 41.11 & 49.07 & 68.44 & 28.97 & 48.13 & 52.33 & 75.12 & 61.08 \\
& MissForest & 46.97 & 39.40 & 40.85 & 52.96 & 27.16 & 41.78 & 59.47 & 64.13 & 45.76 \\
& HyperImpute & 35.89 & 33.22 & 32.85 & 45.21 & 12.41 & 38.78 & 40.17 & 60.45 & 42.56 \\
& DIFFPUTER & 40.51 & 30.49 & 32.27 & 47.67 & 18.60 & 36.75 & 33.67 & 54.89 & 40.62 \\
& ReMasker & 35.02 & 51.23 & 30.53 & 62.61 & 12.41 & 35.13 & 32.54 & 54.92 & 42.51 \\
& Ours & 35.76 & 26.60 & 28.56 & 44.37 & 12.35 & 33.23 & 36.24 & 54.98 & 38.35 \\

\midrule

\multirow{13}{*}{\rotatebox[origin=c]{90}{\textbf{RMSE}}} 
& EM & 66.20 & 105.82 & 81.38 & 84.58 & 29.77 & 78.71 & 77.64 & 125.89 & 90.71 \\
& MIWAE & 122.70 & 141.61 & 126.90 & 116.83 & 104.03 & 118.22 & 98.91 & 134.54 & 127.10 \\
& GAIN & 120.30 & 173.41 & 115.03 & 98.81 & 80.45 & 129.35 & 91.12 & 204.67 & 128.29 \\
& SoftImpute & 91.07 & 118.23 & 115.37 & 104.72 & 56.43 & 112.22 & 89.00 & 133.91 & 112.00 \\
& MICE & 88.96 & 277.87 & 104.27 & 112.58 & 46.54 & 102.62 & 107.10 & 161.08 & 119.39 \\
& MIRACLE & 86.66 & 128.07 & 100.06 & 91.63 & 62.21 & 133.61 & 107.30 & 136.35 & 151.83 \\
& KNN & 120.70 & 125.09 & 106.27 & 105.05 & 75.02 & 104.61 & 75.34 & 143.94 & 124.90 \\
& MissForest & 77.79 & 117.71 & 91.84 & 83.60 & 46.48 & 87.47 & 82.01 & 119.39 & 99.29 \\
& HyperImpute & 66.60 & 98.36 & 82.06 & 78.49 & 32.23 & 79.74 & 57.91 & 125.46 & 95.61 \\
& DIFFPUTER & 76.22 & 110.96 & 87.75 & 79.41 & 64.19 & 86.87 & 52.02 & 123.30 & 101.04 \\
& ReMasker & 63.31 & 115.41 & 83.42 & 91.78 & 30.87 & 78.72 & 46.75 & 118.61 & 93.20 \\
& Ours & 66.35 & 104.95 & 79.72 & 76.38 & 33.86 & 79.69 & 53.66 & 122.22 & 92.64 \\

\bottomrule
\end{tabular}
\caption{In-sample imputation performance under the MNAR setting. Average results are represented by MAE ($\downarrow$) and RMSE ($\downarrow$).}
\label{tab:in_sample_mnar}
\end{table*}

\begin{table*}
\centering

\begin{tabular}{llccccccccc}
\toprule
& \textbf{Method} & \textbf{California} & \textbf{Default} &\textbf{News} &\textbf{Magic} & \textbf{Bean} & \textbf{Gesture} & \textbf{Letter} & \textbf{Adult} & \textbf{Shoppers}  \\
\midrule

\multirow{13}{*}{\rotatebox[origin=c]{90}{\textbf{MAE}}} 
& EM & 39.42 & 49.37 & 58.75 & 54.62 & 10.70 & 38.88 & 56.81 & 62.07 & 44.57 \\
& MIWAE & 81.54 & 54.57 & 64.16 & 83.42 & 63.05 & 61.50 & 73.90 & 61.09 & 50.71 \\
& GAIN & 74.66 & 75.88 & 77.64 & 64.35 & 45.31 & 73.93 & 69.37 & 126.15 & 52.02 \\
& SoftImpute & 58.62 & 48.20 & 72.76 & 64.48 & 34.02 & 57.22 & 66.15 & 74.51 & 61.15 \\
& MICE & 57.82 & 61.77 & 89.26 & 79.59 & 20.67 & 68.31 & 81.50 & 100.05 & 78.50 \\
& MIRACLE & 42.90 & 44.49 & 65.45 & 52.78 & 15.79 & 67.54 & 70.61 & 74.05 & 69.10 \\
& KNN & 72.44 & 33.76 & 47.16 & 65.62 & 24.05 & 44.81 & 49.35 & 68.51 & 55.99 \\
& MissForest & 45.34 & 37.05 & 40.20 & 54.94 & 27.17 & 40.37 & 59.80 & 66.27 & 44.01 \\
& HyperImpute & 37.34 & 31.92 & 29.78 & 47.91 & 12.88 & 41.56 & 41.64 & 60.18 & 42.83 \\
& DIFFPUTER & 37.49 & 26.28 & 30.77 & 48.72 & 14.30 & 33.00 & 32.31 & 50.68 & 37.31 \\
& ReMasker & 34.67 & 49.97 & 29.86 & 64.94 & 12.48 & 35.30 & 33.81 & 52.84 & 41.78 \\
& Ours & 35.16 & 24.29 & 28.49 & 46.85 & 12.58 & 33.58 & 37.86 & 54.03 & 36.58 \\

\midrule

\multirow{13}{*}{\rotatebox[origin=c]{90}{\textbf{RMSE}}} 
& EM & 62.72 & 121.33 & 320.31 & 85.77 & 28.30 & 75.03 & 77.32 & 109.85 & 82.36 \\
& MIWAE & 116.24 & 111.25 & 110.31 & 115.69 & 99.33 & 110.16 & 98.80 & 118.90 & 112.33 \\
& GAIN & 105.98 & 137.54 & 141.46 & 93.84 & 70.49 & 118.22 & 92.46 & 189.74 & 110.33 \\
& SoftImpute & 84.02 & 95.46 & 144.53 & 97.99 & 50.85 & 103.86 & 88.48 & 128.59 & 99.62 \\
& MICE & 86.25 & 307.60 & 4516.24 & 112.71 & 46.16 & 98.90 & 106.51 & 145.23 & 112.03 \\
& MIRACLE & 75.73 & 106.14 & 171.99 & 92.55 & 38.17 & 131.66 & 106.32 & 142.76 & 126.62 \\
& KNN & 108.17 & 87.79 & 86.87 & 99.44 & 57.15 & 93.46 & 70.58 & 126.17 & 106.72 \\
& MissForest & 73.22 & 91.32 & 77.03 & 86.99 & 45.86 & 80.46 & 82.64 & 128.87 & 88.64 \\
& HyperImpute & 64.38 & 75.10 & 346.00 & 81.60 & 31.38 & 77.94 & 59.91 & 113.89 & 91.28 \\
& DIFFPUTER & 65.65 & 78.87 & 71.19 & 81.38 & 42.72 & 73.12 & 50.36 & 107.27 & 84.01 \\
& ReMasker & 62.44 & 96.39 & 66.46 & 96.38 & 30.72 & 76.02 & 48.97 & 105.67 & 84.21 \\
& Ours & 62.69 & 69.82 & 64.22 & 80.06 & 33.30 & 74.92 & 55.90 & 109.89 & 80.24 \\

\bottomrule
\end{tabular}
\caption{Out-of-sample imputation performance under the MNAR setting. Average results are represented by MAE ($\downarrow$) and RMSE ($\downarrow$).}
\label{tab:out_mnar}
\end{table*}

\begin{table*}
\centering

\begin{tabular}{lcccc|cccc}
\toprule
\textbf{Method} & \textbf{Default} &\textbf{News} & \textbf{Adult} & \textbf{Shoppers} & \textbf{Default} &\textbf{News} & \textbf{Adult} & \textbf{Shoppers}  \\
& \multicolumn{4}{c}{In-Sample}&\multicolumn{4}{c}{Out-of-Sample}\\
\midrule
EM & 71.09 & 40.91 & 63.38 & 57.23 & 69.94 & 40.93 & 63.09 & 57.13 \\
MIWAE & 47.40 & 17.56 & 52.99 & 49.03 & 48.45 & 17.60 & 52.89 & 49.30 \\
GAIN & 60.75 & 26.87 & 56.11 & 49.50 & 56.89 & 28.54 & 53.38 & 49.12 \\
SoftImpute & 59.20 & 22.42 & 57.31 & 52.52 & 58.87 & 18.89 & 58.33 & 52.32 \\
MICE & 57.90 & 30.18 & 49.93 & 44.05 & 58.32 & 30.07 & 49.72 & 44.01 \\
MIRACLE & 69.40 & 39.23 & 63.21 & 48.06 & 68.57 & 36.02 & 62.19 & 51.48 \\
KNN & 69.53 & 36.95 & 58.83 & 49.75 & 69.86 & 37.57 & 59.10 & 49.82 \\
MissForest & 68.98 & 34.26 & 65.53 & 52.57 & 68.97 & 34.74 & 65.24 & 52.06 \\
HyperImpute & 72.09 & 41.94 & 67.88 & 56.84 & 71.21 & 41.96 & 68.13 & 54.56 \\
DIFFPUTER & 73.25 & 43.28 & 67.90 & 55.83 & 73.84 & 43.74 & 67.86 & 56.53 \\
ReMasker & 75.37 & 46.76 & 72.99 & 56.94 & 75.45 & 47.02 & 72.59 & 56.63 \\
Ours & 75.37 & 46.76 & 72.99 & 57.65 & 75.45 & 47.02 & 72.59 & 57.25 \\

\bottomrule
\end{tabular}
\caption{Categorical imputation performance under the MNAR setting. Average results are represented by accuracy ($\uparrow$).}
\label{tab:mnar_acc}
\end{table*}

\begin{table*}
\centering

\begin{tabular}{llccccccccc}
\toprule
& \textbf{Method} & \textbf{California} & \textbf{Default} &\textbf{News} &\textbf{Magic} & \textbf{Bean} & \textbf{Gesture} & \textbf{Letter} & \textbf{Adult} & \textbf{Shoppers}  \\
\midrule

\multirow{13}{*}{\rotatebox[origin=c]{90}{\textbf{MAE}}} 
& EM & 37.41 & 27.89 & 37.96 & 49.26 & 10.39 & 34.30 & 54.88 & 55.95 & 40.22 \\
& MIWAE & 76.40 & 51.05 & 63.35 & 76.96 & 59.98 & 57.42 & 72.44 & 58.58 & 46.78 \\
& GAIN & 78.08 & 69.36 & 55.94 & 60.04 & 34.62 & 72.62 & 64.26 & 120.61 & 47.70 \\
& SoftImpute & 55.81 & 37.98 & 62.23 & 57.63 & 27.97 & 51.02 & 64.64 & 61.58 & 55.76 \\
& MICE & 56.20 & 59.36 & 62.26 & 75.82 & 17.25 & 65.58 & 80.73 & 98.03 & 75.75 \\
& MIRACLE & 38.00 & 37.94 & 38.92 & 43.63 & 10.93 & 68.16 & 65.77 & 58.87 & 74.01 \\
& KNN & 64.37 & 30.68 & 45.66 & 57.52 & 20.28 & 41.13 & 47.84 & 61.87 & 48.42 \\
& MissForest & 42.77 & 32.62 & 39.16 & 50.16 & 25.30 & 36.66 & 58.50 & 60.47 & 41.46 \\
& HyperImpute & 32.61 & 27.83 & 27.77 & 42.32 & 11.35 & 36.28 & 39.82 & 54.77 & 37.85 \\
& DIFFPUTER & 34.67 & 23.20 & 29.36 & 43.22 & 11.77 & 29.23 & 30.60 & 48.77 & 34.65 \\
& ReMasker & 32.02 & 38.97 & 29.33 & 56.80 & 12.07 & 31.09 & 32.61 & 50.28 & 39.37 \\
& Ours & 32.03 & 22.24 & 27.29 & 41.71 & 11.44 & 28.62 & 35.58 & 49.69 & 34.12 \\

\midrule

\multirow{13}{*}{\rotatebox[origin=c]{90}{\textbf{RMSE}}} 
& EM & 58.28 & 68.80 & 63.78 & 74.25 & 27.93 & 68.39 & 74.50 & 94.13 & 76.71 \\
& MIWAE & 105.18 & 104.09 & 105.73 & 105.19 & 93.34 & 104.77 & 96.85 & 104.32 & 106.45 \\
& GAIN & 106.65 & 118.31 & 90.85 & 84.64 & 54.37 & 114.28 & 85.42 & 176.70 & 105.24 \\
& SoftImpute & 80.59 & 78.53 & 92.14 & 84.30 & 47.40 & 95.55 & 86.23 & 100.32 & 91.05 \\
& MICE & 82.49 & 95.76 & 90.07 & 105.18 & 41.13 & 94.66 & 105.32 & 132.77 & 108.35 \\
& MIRACLE & 64.57 & 99.00 & 77.51 & 71.50 & 30.48 & 118.80 & 98.70 & 104.27 & 140.18 \\
& KNN & 93.89 & 80.78 & 79.85 & 85.16 & 45.77 & 88.76 & 67.77 & 104.14 & 96.20 \\
& MissForest & 65.07 & 78.90 & 69.56 & 76.43 & 42.72 & 75.35 & 80.33 & 103.97 & 85.10 \\
& HyperImpute & 56.22 & 68.44 & 57.75 & 70.02 & 29.92 & 71.78 & 57.60 & 95.08 & 80.14 \\
& DIFFPUTER & 57.17 & 68.14 & 58.92 & 69.35 & 32.64 & 67.60 & 47.27 & 92.20 & 78.08 \\
& ReMasker & 52.49 & 80.61 & 58.65 & 82.29 & 30.48 & 67.56 & 46.86 & 89.42 & 78.37 \\
& Ours & 54.50 & 65.20 & 56.55 & 69.07 & 32.45 & 67.93 & 52.62 & 92.53 & 77.56 \\

\bottomrule
\end{tabular}
\caption{In-sample imputation performance under the MCAR setting. Average results are represented by MAE ($\downarrow$) and RMSE ($\downarrow$).}
\label{tab:in_mcar}
\end{table*}

\begin{table*}
\centering

\begin{tabular}{llccccccccc}
\toprule
& \textbf{Method} & \textbf{California} & \textbf{Default} &\textbf{News} &\textbf{Magic} & \textbf{Bean} & \textbf{Gesture} & \textbf{Letter} & \textbf{Adult} & \textbf{Shoppers}  \\
\midrule

\multirow{13}{*}{\rotatebox[origin=c]{90}{\textbf{MAE}}} 
& EM & 37.31 & 27.54 & 55.91 & 49.38 & 10.36 & 34.47 & 55.57 & 56.86 & 40.93 \\
& MIWAE & 76.85 & 51.28 & 62.59 & 76.81 & 59.96 & 56.39 & 72.94 & 59.12 & 47.32 \\
& GAIN & 69.47 & 62.82 & 72.57 & 58.36 & 37.21 & 61.17 & 66.68 & 122.98 & 46.19 \\
& SoftImpute & 55.13 & 40.89 & 70.03 & 57.03 & 27.57 & 51.06 & 64.96 & 63.08 & 55.83 \\
& MICE & 56.16 & 59.05 & 62.14 & 75.57 & 17.15 & 65.05 & 81.14 & 97.93 & 76.23 \\
& MIRACLE & 39.47 & 48.45 & 61.91 & 44.19 & 14.04 & 58.03 & 67.78 & 64.12 & 63.17 \\
& KNN & 64.63 & 30.25 & 45.52 & 57.63 & 20.35 & 40.27 & 47.99 & 62.02 & 48.84 \\
& MissForest & 42.68 & 32.44 & 39.01 & 50.28 & 25.49 & 35.89 & 58.84 & 60.66 & 42.20 \\
& HyperImpute & 33.65 & 28.95 & 33.70 & 43.08 & 12.07 & 37.94 & 41.35 & 56.29 & 40.41 \\
& DIFFPUTER & 34.65 & 22.72 & 29.11 & 43.27 & 11.72 & 27.73 & 30.30 & 48.42 & 34.62 \\
& ReMasker & 32.03 & 38.82 & 29.14 & 56.89 & 11.99 & 30.71 & 32.85 & 50.25 & 39.91 \\
& Ours & 32.93 & 22.42 & 27.75 & 42.48 & 11.77 & 29.41 & 37.10 & 51.13 & 34.81 \\

\midrule

\multirow{13}{*}{\rotatebox[origin=c]{90}{\textbf{RMSE}}} 
& EM & 58.14 & 64.84 & 269.89 & 74.63 & 27.82 & 65.78 & 75.40 & 94.40 & 76.46 \\
& MIWAE & 106.27 & 101.02 & 102.89 & 105.11 & 92.59 & 100.03 & 97.49 & 104.89 & 104.15 \\
& GAIN & 95.88 & 117.46 & 125.88 & 82.78 & 58.21 & 96.43 & 89.02 & 177.51 & 99.45 \\
& SoftImpute & 79.14 & 79.81 & 130.88 & 83.28 & 45.60 & 91.18 & 86.79 & 100.86 & 89.78 \\
& MICE & 82.39 & 93.22 & 88.90 & 104.87 & 40.87 & 91.93 & 105.95 & 132.38 & 108.28 \\
& MIRACLE & 66.22 & 107.74 & 176.60 & 71.85 & 32.35 & 113.53 & 101.89 & 109.06 & 112.67 \\
& KNN & 94.91 & 76.21 & 77.66 & 85.43 & 44.42 & 83.50 & 67.99 & 104.43 & 93.53 \\
& MissForest & 64.76 & 75.28 & 67.95 & 76.87 & 42.26 & 70.18 & 80.99 & 104.27 & 85.05 \\
& HyperImpute & 57.41 & 66.07 & 891.86 & 70.68 & 30.44 & 70.06 & 59.68 & 96.72 & 81.47 \\
& DIFFPUTER & 57.18 & 63.73 & 59.44 & 69.61 & 31.87 & 59.40 & 47.09 & 92.04 & 76.31 \\
& ReMasker & 52.58 & 78.00 & 56.89 & 82.58 & 29.88 & 63.78 & 47.19 & 88.96 & 77.48 \\
& Ours & 56.36 & 61.66 & 56.77 & 70.46 & 32.54 & 64.58 & 54.66 & 94.39 & 76.58 \\

\bottomrule
\end{tabular}
\caption{Out-of-sample imputation performance under the MCAR setting. Average results are represented by MAE ($\downarrow$) and RMSE ($\downarrow$).}
\label{tab:out_mcar}
\end{table*}
\begin{table*}
\centering

\begin{tabular}{lcccc|cccc}
\toprule
\textbf{Method} & \textbf{Default} &\textbf{News} & \textbf{Adult} & \textbf{Shoppers} & \textbf{Default} &\textbf{News} & \textbf{Adult} & \textbf{Shoppers}  \\
& \multicolumn{4}{c}{In-Sample}&\multicolumn{4}{c}{Out-of-Sample}\\
\midrule
EM & 71.11 & 41.49 & 63.10 & 57.34 & 71.32 & 40.82 & 63.07 & 57.13 \\
MIWAE & 47.91 & 17.86 & 52.73 & 49.14 & 48.32 & 17.85 & 52.85 & 49.26 \\
GAIN & 60.78 & 27.21 & 56.52 & 50.93 & 57.72 & 29.99 & 53.29 & 49.49 \\
SoftImpute & 59.72 & 18.75 & 57.89 & 52.62 & 59.52 & 19.17 & 58.43 & 52.28 \\
MICE & 57.65 & 30.33 & 49.83 & 43.82 & 57.92 & 30.37 & 49.76 & 44.10 \\
MIRACLE & 69.42 & 39.99 & 62.86 & 48.60 & 66.80 & 36.13 & 62.36 & 52.56 \\
KNN & 69.54 & 37.52 & 59.40 & 50.10 & 69.75 & 36.87 & 59.29 & 50.13 \\
MissForest & 68.38 & 35.13 & 66.21 & 53.31 & 68.70 & 34.86 & 65.89 & 53.30 \\
HyperImpute & 72.29 & 41.59 & 68.16 & 58.51 & 71.02 & 40.57 & 68.81 & 53.98 \\
DIFFPUTER & 73.40 & 43.64 & 67.96 & 56.04 & 73.61 & 43.47 & 68.00 & 56.36 \\
ReMasker & 75.69 & 47.51 & 72.93 & 56.89 & 75.83 & 46.74 & 72.67 & 56.93 \\
Ours & 76.06 & 47.50 & 72.93 & 57.75 & 76.22 & 46.72 & 72.67 & 57.22 \\

\bottomrule
\end{tabular}
\caption{Categorical imputation performance under the MCAR setting. Average results are represented by accuracy ($\uparrow$).}
\label{tab:mcar_acc}
\end{table*}

\begin{table*}
\centering

\begin{tabular}{llccccccccc}
\toprule
& \textbf{Method} & \textbf{California} & \textbf{Default} &\textbf{News} &\textbf{Magic} & \textbf{Bean} & \textbf{Gesture} & \textbf{Letter} & \textbf{Adult} & \textbf{Shoppers}  \\
\midrule

\multirow{13}{*}{\rotatebox[origin=c]{90}{\textbf{MAE}}} 
& EM & 39.65 & 31.78 & 39.33 & 54.16 & 12.59 & 37.18 & 56.42 & 64.69 & 47.38 \\
& MIWAE & 85.94 & 57.13 & 65.25 & 83.39 & 70.00 & 61.51 & 72.27 & 63.94 & 54.69 \\
& GAIN & 84.34 & 84.63 & 60.34 & 71.14 & 60.97 & 125.69 & 81.42 & 119.79 & 57.29 \\
& SoftImpute & 64.38 & 47.54 & 63.59 & 65.90 & 44.32 & 57.62 & 66.38 & 70.95 & 66.52 \\
& MICE & 57.03 & 62.64 & 63.56 & 80.18 & 22.81 & 67.69 & 81.71 & 108.05 & 81.52 \\
& MIRACLE & 47.46 & 47.38 & 41.81 & 63.00 & 23.03 & 74.91 & 73.38 & 90.53 & 73.79 \\
& KNN & 48.02 & 31.74 & 45.27 & 51.50 & 19.11 & 41.47 & 38.08 & 65.25 & 48.58 \\
& MissForest & 45.72 & 42.94 & 39.40 & 52.95 & 28.85 & 39.12 & 60.64 & 78.34 & 45.94 \\
& HyperImpute & 35.61 & 30.36 & 30.55 & 46.93 & 13.82 & 37.32 & 41.77 & 66.46 & 44.74 \\
& DIFFPUTER & 40.80 & 30.29 & 32.62 & 49.41 & 24.08 & 37.37 & 37.47 & 57.18 & 43.18 \\
& ReMasker & 34.30 & 40.23 & 29.99 & 58.31 & 13.11 & 34.42 & 32.41 & 57.31 & 45.48 \\
& Ours & 34.38 & 25.27 & 27.54 & 45.51 & 14.25 & 31.03 & 36.70 & 56.61 & 39.35 \\

\midrule

\multirow{13}{*}{\rotatebox[origin=c]{90}{\textbf{RMSE}}} 
& EM & 61.56 & 95.23 & 71.14 & 85.09 & 32.89 & 75.97 & 76.72 & 147.36 & 97.75 \\
& MIWAE & 125.03 & 132.83 & 118.92 & 115.99 & 114.81 & 115.57 & 96.63 & 154.48 & 128.22 \\
& GAIN & 121.54 & 168.43 & 109.15 & 105.19 & 96.72 & 185.64 & 109.22 & 215.97 & 130.69 \\
& SoftImpute & 91.49 & 104.97 & 106.86 & 102.26 & 67.16 & 110.22 & 88.13 & 152.09 & 116.86 \\
& MICE & 83.97 & 117.00 & 95.65 & 113.51 & 47.84 & 100.38 & 106.64 & 183.11 & 125.01 \\
& MIRACLE & 86.34 & 121.34 & 91.57 & 110.69 & 72.27 & 127.05 & 106.16 & 203.04 & 142.15 \\
& KNN & 77.12 & 98.70 & 88.48 & 83.79 & 42.78 & 90.57 & 56.41 & 153.26 & 104.73 \\
& MissForest & 74.58 & 123.56 & 80.05 & 85.39 & 50.21 & 83.00 & 83.08 & 167.86 & 101.72 \\
& HyperImpute & 64.62 & 87.76 & 70.40 & 81.35 & 34.02 & 76.64 & 59.94 & 152.97 & 99.57 \\
& DIFFPUTER & 78.46 & 104.79 & 79.66 & 81.15 & 79.35 & 88.77 & 56.16 & 144.84 & 105.38 \\
& ReMasker & 59.67 & 98.52 & 71.88 & 87.99 & 32.26 & 77.37 & 46.55 & 140.06 & 95.08 \\
& Ours & 62.60 & 87.86 & 68.75 & 78.75 & 38.46 & 75.42 & 54.07 & 144.79 & 93.29 \\

\bottomrule
\end{tabular}
\caption{In-sample imputation performance under the MAR setting. Average results are represented by MAE ($\downarrow$) and RMSE ($\downarrow$).}
\label{tab:in_mar}
\end{table*}

\begin{table*}
\centering

\begin{tabular}{llccccccccc}
\toprule
& \textbf{Method} & \textbf{California} & \textbf{Default} &\textbf{News} &\textbf{Magic} & \textbf{Bean} & \textbf{Gesture} & \textbf{Letter} & \textbf{Adult} & \textbf{Shoppers}  \\
\midrule

\multirow{13}{*}{\rotatebox[origin=c]{90}{\textbf{MAE}}} 
& EM & 37.21 & 33.41 & 54.13 & 51.22 & 11.01 & 38.58 & 57.19 & 65.79 & 44.97 \\
& MIWAE & 79.52 & 54.55 & 65.79 & 78.64 & 69.63 & 59.08 & 74.08 & 63.05 & 48.30 \\
& GAIN & 72.35 & 80.31 & 69.56 & 69.66 & 64.96 & 98.46 & 73.42 & 116.57 & 47.72 \\
& SoftImpute & 61.58 & 48.76 & 70.14 & 60.66 & 38.26 & 56.70 & 66.49 & 70.03 & 61.10 \\
& MICE & 54.89 & 62.03 & 107.26 & 76.65 & 23.09 & 68.09 & 82.24 & 105.74 & 78.53 \\
& MIRACLE & 43.09 & 54.67 & 66.77 & 59.07 & 17.66 & 66.08 & 77.42 & 113.07 & 79.69 \\
& KNN & 59.41 & 31.38 & 47.02 & 57.60 & 19.27 & 40.83 & 45.74 & 67.57 & 49.27 \\
& MissForest & 42.66 & 37.34 & 41.55 & 51.04 & 27.51 & 39.40 & 59.29 & 69.83 & 45.85 \\
& HyperImpute & 34.82 & 32.37 & 32.34 & 45.75 & 13.33 & 40.25 & 43.05 & 66.02 & 42.11 \\
& DIFFPUTER & 35.46 & 27.05 & 32.42 & 46.14 & 15.32 & 33.27 & 34.56 & 56.54 & 37.12 \\
& ReMasker & 32.81 & 39.96 & 31.24 & 56.62 & 11.94 & 35.45 & 34.36 & 58.96 & 41.18 \\
& Ours & 32.31 & 27.16 & 29.93 & 44.50 & 12.94 & 32.56 & 38.23 & 55.90 & 36.34 \\

\midrule

\multirow{13}{*}{\rotatebox[origin=c]{90}{\textbf{RMSE}}} 
& EM & 58.18 & 86.28 & 187.87 & 79.34 & 29.93 & 74.16 & 78.46 & 142.96 & 82.54 \\
& MIWAE & 112.12 & 115.24 & 111.98 & 107.92 & 105.08 & 106.06 & 99.01 & 147.02 & 109.76 \\
& GAIN & 103.73 & 157.21 & 125.57 & 100.94 & 94.64 & 153.16 & 97.53 & 204.64 & 104.45 \\
& SoftImpute & 86.16 & 101.39 & 121.47 & 89.89 & 55.11 & 103.80 & 89.08 & 145.62 & 99.88 \\
& MICE & 80.42 & 109.38 & 8728.05 & 106.85 & 49.11 & 97.48 & 107.43 & 174.15 & 113.68 \\
& MIRACLE & 75.53 & 120.09 & 116.89 & 99.96 & 45.74 & 122.22 & 112.88 & 217.21 & 141.11 \\
& KNN & 90.39 & 89.48 & 83.86 & 86.46 & 44.63 & 84.54 & 67.27 & 147.35 & 96.40 \\
& MissForest & 67.94 & 101.63 & 77.69 & 78.99 & 46.18 & 78.20 & 82.05 & 156.33 & 92.22 \\
& HyperImpute & 59.22 & 83.91 & 379.56 & 75.92 & 32.69 & 75.72 & 61.62 & 147.92 & 88.29 \\
& DIFFPUTER & 61.92 & 83.50 & 70.23 & 74.31 & 45.61 & 73.76 & 53.45 & 138.65 & 83.42 \\
& ReMasker & 57.26 & 91.79 & 68.16 & 83.31 & 29.26 & 74.93 & 49.83 & 136.54 & 83.06 \\
& Ours & 57.79 & 84.06 & 65.24 & 73.35 & 34.85 & 72.21 & 56.50 & 138.74 & 81.69 \\

\bottomrule
\end{tabular}
\caption{Out-of-sample imputation performance under the MAR setting. Average results are represented by MAE ($\downarrow$) and RMSE ($\downarrow$).}
\label{tab:out_mar}
\end{table*}

\begin{table*}
\centering

\begin{tabular}{lcccc|cccc}
\toprule
\textbf{Method} & \textbf{Default} &\textbf{News} & \textbf{Adult} & \textbf{Shoppers} & \textbf{Default} &\textbf{News} & \textbf{Adult} & \textbf{Shoppers}  \\
& \multicolumn{4}{c}{In-Sample}&\multicolumn{4}{c}{Out-of-Sample}\\
\midrule
EM & 71.29 & 28.32 & 59.27 & 60.19 & 70.89 & 41.56 & 60.41 & 58.37 \\
MIWAE & 46.31 & 14.33 & 49.01 & 52.62 & 49.07 & 15.96 & 52.14 & 50.78 \\
GAIN & 57.07 & 19.50 & 50.81 & 52.56 & 54.04 & 30.32 & 53.96 & 50.54 \\
SoftImpute & 59.38 & 15.17 & 53.46 & 54.69 & 59.62 & 17.42 & 57.02 & 53.30 \\
MICE & 58.64 & 21.33 & 46.49 & 46.87 & 58.50 & 29.77 & 49.20 & 45.91 \\
MIRACLE & 67.04 & 24.74 & 57.48 & 50.08 & 65.53 & 31.63 & 59.36 & 49.19 \\
KNN & 71.07 & 26.59 & 60.83 & 55.00 & 70.67 & 37.37 & 61.95 & 52.62 \\
MissForest & 69.23 & 25.27 & 61.05 & 53.45 & 70.41 & 35.24 & 67.82 & 56.73 \\
HyperImpute & 72.12 & 26.63 & 65.73 & 53.20 & 69.56 & 38.57 & 70.89 & 49.13 \\
DIFFPUTER & 73.58 & 30.15 & 65.77 & 59.45 & 73.16 & 43.24 & 68.68 & 58.33 \\
ReMasker & 77.14 & 31.78 & 70.87 & 59.72 & 76.55 & 47.94 & 74.77 & 58.47 \\
Ours & 77.14 & 31.75 & 70.87 & 60.45 & 76.55 & 47.90 & 74.77 & 58.16 \\

\bottomrule
\end{tabular}
\caption{Categorical imputation performance under the MAR setting. Average results are represented by accuracy ($\uparrow$).}
\label{tab:mar_acc}
\end{table*}

\end{document}